\documentclass{article} 
\usepackage{iclr2027_conference_arxiv,times}

\usepackage{amsmath,amsfonts,bm}

\def\eqref#1{equation~\ref{#1}}

\def\1{\bm{1}}

\DeclareMathAlphabet{\mathsfit}{\encodingdefault}{\sfdefault}{m}{sl}
\SetMathAlphabet{\mathsfit}{bold}{\encodingdefault}{\sfdefault}{bx}{n}

\usepackage{url}
\usepackage{amsmath, amssymb, amsthm, mathtools, bm}
\usepackage{booktabs}
\usepackage{enumitem}
\usepackage{xcolor}
\usepackage{hyperref}
\usepackage{algorithm}
\usepackage{algpseudocode}
\usepackage{tikz}

\hypersetup{
    colorlinks=true,
    linkcolor=blue,
    citecolor=blue,
    urlcolor=blue
}
\usepackage{tabularx}

\newtheorem{proposition}{Proposition}
\newtheorem{theorem}{Theorem}

\newtheorem{assumption}{Assumption}
\newtheorem{corollary}{Corollary}
\newtheorem{lemma}{Lemma}

\usepackage{wrapfig}
\usepackage{xcolor}
\usepackage[most]{tcolorbox}

\newtcolorbox{takeawaybox}{
    colback=gray!10,
    colframe=gray!45,
    boxrule=0.6pt,
    arc=6pt,
    outer arc=6pt,
    left=7pt,
    right=7pt,
    top=6pt,
    bottom=6pt
}

\iclrfinalcopy

\title{Reliable Replay through Spatial Coherence in Online Continual Learning}

\iclrfinalcopy

\author{Haixiang Sun, Jiefu Zhang, Yinghao He, Yang Xu, Vaneet Aggarwal, Bharat Bhargava, Andrew L. Liu\\
Purdue University\\
West Lafayette, IN 47907, USA \\
\texttt{\{sun1321, zhan4018, he923, xu1720, vaneet, bbshail, andrewliu\}@purdue.edu} \\
}

\begin{document}

\maketitle

\begin{abstract}
Continually adapting models to new tasks requires retaining earlier knowledge under limited memory and computation. Experience replay addresses this challenge, but priorities based on individual loss increases overlook how related memories respond to the same update and can overemphasize isolated responses. We introduce \emph{SPatial coHErent risk control for REplay} (SPHERE), a general replay-allocation method applicable across a broad range of learning settings. SPHERE uses a representation kernel to aggregate signed prospective loss changes, attenuating unsupported spikes while retaining coherent increases. It then formulates allocation as entropy-regularized transport, redistributing uniform source mass toward supported high-risk regions while penalizing long-distance transfers. We derive replay coefficients from the transport objective's sensitivity to the original loss changes and blend them with uniform replay to maintain baseline rehearsal. Our analysis establishes conditions under which kernel aggregation improves risk estimation and bounds transport-value inflation due to residual noise and smoothing bias. Experiments demonstrate that SPHERE improves accuracy and reduces forgetting across noisy-label vision tasks, continual language-model instruction tuning, and code-generation reinforcement learning with incomplete test rewards. 
\end{abstract}

\section{Introduction}
\label{sec:introduction}


Online continual learning requires models to acquire new capabilities while retaining those learned earlier. The requirement spans supervised learning and language-model post-training, including instruction tuning and reinforcement learning for generation tasks. Although their training objectives differ, sequential updates can improve performance on current tasks while degrading earlier capabilities. Experience replay \citep{rolnick2019experience} addresses this tension by revisiting stored examples or problems alongside new data \citep{kirkpatrick2017overcoming,li2017learning,buzzega2020dark}. However, limited memory and computation allow only part of the history to be rehearsed at each update. Effective replay therefore depends on how rehearsal is allocated across past experience under the learning objective in use.



A common allocation signal is the loss increase that a stored example would experience after a virtual update on the current batch. The same construction can use a supervised loss or a policy-learning surrogate. However, treating each increase independently overlooks whether related memories are affected together. Interference-based retrieval methods use the loss increase to assess a memory's vulnerability and guide replay prioritization \citep{aljundi2019online,li2024adaer}. Such pointwise increments quantify how strongly individual memories respond, but not whether nearby memories respond similarly. Under unreliable supervision, the increase may also reflect conflict with an incorrect target rather than a decline in actual task performance. This raises a complementary question:

\emph{When does an observed loss increase justify higher replay priority?}

Our insight is to assess each loss increase using nearby memories' responses to the same update. When the underlying response is locally smooth, \emph{spatially coherent increases} support a shared disturbance, whereas an \emph{isolated spike} surrounded by unchanged or decreasing losses has weaker support. Moderate but mutually supported increases can therefore receive higher replay priority than a larger isolated response. Spatial information thus changes the risk criterion: an increment is assessed not only by its magnitude, but also by its agreement with neighboring responses.

\begin{figure}[t]
    \centering
    \includegraphics[width=0.85\linewidth]{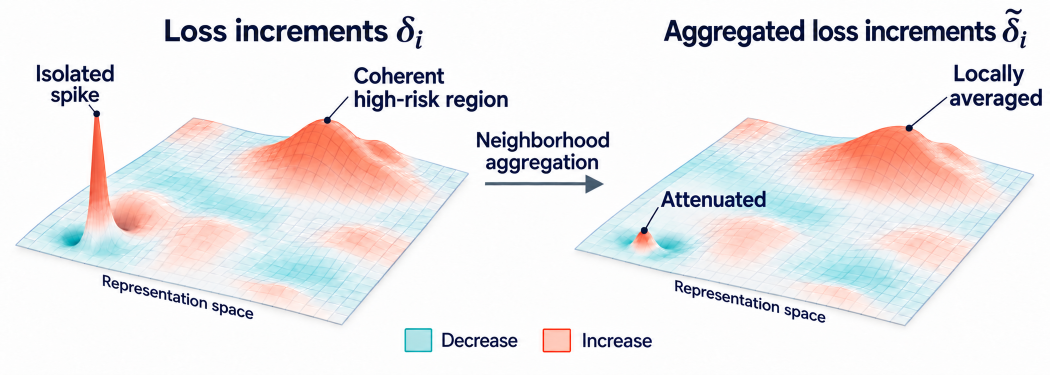}
    \caption{Isolated Spike and Coherent Forgetting.}
    \label{fig:neighborhood}
\end{figure}

Label noise is exactly one possible source of such isolated spike. An update may increase loss against a corrupted target while improving prediction of the true class. Prioritizing such increases may then reinforce incorrect supervision. Existing methods reduce the influence of noisy labels through purification and robust training \citep{kim2021continual,bang2022online,karim2022cnll,millunzi2024may}. However, if such corrupted examples remain in memory, their loss increases can be mistaken for forgetting that needs to be prevented. Extra replay can then reinforce incorrect targets and consume capacity needed to preserve useful knowledge.
A similar problem arises in reinforcement learning with imperfect rewards. In code generation, incomplete unit tests can reward incorrect programs \citep{le2022coderl}. If a later update reduces the likelihood of such programs, using replay to reverse that change could preserve faulty behavior. To address this problem, spatial aggregation can reduce the influence of isolated loss increases before replay allocation. It checks whether nearby memories show similar loss changes under the same update rather than directly identifying changes in clean-label prediction.


Building on this insight, we introduce \emph{SPatial coHErent risk control for REplay} (SPHERE), a general replay-allocation method. SPHERE aggregates loss changes within representation neighborhoods to attenuate unsupported spikes and retain coherent increases. Entropy-regularized transport then redistributes uniform mass toward supported high-risk regions while penalizing long-distance transfers. The resulting replay coefficients are blended with uniform replay to maintain baseline rehearsal. SPHERE is broadly applicable across tasks and training pipelines, without being tied to a particular model architecture or learning objective. Experiments show gains in accuracy and retention across noisy-label vision tasks, continual LLM instruction tuning, and code-generation reinforcement learning with incomplete test rewards. Our contributions are as three-fold:

\begin{itemize}[leftmargin=*,itemsep=1pt,topsep=2pt]
    \item We introduce SPHERE, a general replay-allocation method applicable across a broad range of learning settings. It combines neighborhood aggregation with entropy-regularized transport to direct replay toward supported high-risk regions while penalizing long-distance mass transfers.

    \item We derive replay coefficients from the transport objective's sensitivity to the original loss changes, linking risk estimation to replay allocation. We establish conditions under which neighborhood aggregation improves risk estimates and bound transport-value inflation due to residual noise and smoothing bias.

    \item We demonstrate gains in accuracy and retention across noisy-label vision tasks and continual LLM post-training, including code-generation reinforcement learning with incomplete test rewards. Component ablations support the roles of both neighborhood information and transport-based allocation.
\end{itemize}

\section{Related Works}

\paragraph{Online Continual Learning.}
Online continual learning (OCL) adapts to nonstationary streams in one pass with bounded memory, often without task boundaries. Task-free formulations and memory-based replay provide foundations \citep{aljundi2019task,rolnick2019experience,pmlr-v119-chrysakis20a,GDumb}. Work on memory selection and retrieval uses prospective interference, Shapley valuation, and memory editing \citep{aljundi2019online,shim2021online,jin2021gradient}. Other methods change replay representations and objectives through dark knowledge, contrastive learning, asymmetric classification, and evolving prototypes \citep{buzzega2020dark,mai2021supervised,caccia2022new,de2021continual}, as well as mutual information, complementary learning systems, and extended distillation \citep{pmlr-v162-guo22g,arani2022learning,boschini2022class}. OCL has also expanded to lifelong language learning and parameter-efficient continual instruction tuning \citep{de2019episodic,wang2023orthogonal}, while recent studies examine realistic online evaluation and generalization challenges \citep{ghunaim2023real,wang2024forgetting}. Building on prospective interference, SPHERE assesses whether nearby memories show similar signed loss changes under the same update before assigning extra replay priority.

\paragraph{Catastrophic Forgetting.}
Catastrophic interference is the loss of earlier competence after sequential updates \citep{MCCLOSKEY1989109,goodfellow2013empirical}. Regularization protects important parameters through Fisher or synaptic importance \citep{kirkpatrick2017overcoming,zenke2017continual,aljundi2018memory}, or preserves model functions through distillation and functional constraints \citep{li2017learning,pan2020continual}. Gradient-based methods use episodic constraints and meta-replay \citep{lopez2017gradient,chaudhry2018efficient,riemer2018learning}, or restrict updates to orthogonal subspaces \citep{farajtabar2020orthogonal,saha2021gradient}. Architectural methods isolate or expand capacity \citep{rusu2016progressive,serra2018overcoming,yoon2018lifelong}. Diagnostic studies connect forgetting to example-level events, training regimes, task semantics, and representation change \citep{toneva2018an,mirzadeh2020understanding,ramasesh2021anatomy}. Under noisy supervision, methods also adapt memory management or training to unreliable labels \citep{kim2021continual,bang2022online,karim2022cnll,millunzi2024may}. SPHERE focuses on whether a prospective loss increase has support from nearby memories, since an isolated increase may be a poor reason to allocate more replay.

\paragraph{Transport-Based Allocation.}
Classical optimal transport (OT) couples two prescribed marginals at minimum cost, and entropic regularization makes the coupling smooth and tractable \citep{cuturi2013sinkhorn,genevay2016stochastic,sun2026robust}. In continual learning, transport has been used to evolve replay memories and retain diverse stored examples \citep{wang2022improving,ye2025continual}, align event-class predictions with pretrained language-model information \citep{dao2024lifelong}, and transport classifiers between old and new classes \citep{zhou2021co}. These approaches leave open how to distinguish loss increases shared by neighboring memories from isolated responses when allocating replay. SPHERE addresses this problem by aggregating neighboring loss changes before transport allocation, directing replay toward supported forgetting risks while limiting the influence of isolated spikes.

\section{Replay-Priority Reliability}
\label{sec:noisy-risk-field}

Replay allocates limited training capacity to previously observed examples. The key decision is which memories need additional rehearsal under the current update. At step $t$, the learner receives a mini-batch $B_t$ and maintains a bounded memory $\mathcal M_t=\{z_i\}_{i=1}^{m}$ with $m\geq1$, where $z_i=(x_i,\widetilde y_i)$ records an input and its stored supervision. Let $L_t(\theta)$ be the current-batch loss, $\ell(z_i;\theta)$ the memory loss, and $\lambda\geq0$ the replay strength. This formulation accommodates different tasks and training pipelines, including supervised learning, contrasive loss in LLM, and reinforcement learning. Uniform experience replay (ER) minimizes
\begin{equation}
L_t(\theta)+\frac{\lambda}{m}\sum_{i=1}^{m}\ell(z_i;\theta).
\label{eq:er-objective}
\end{equation}
To assess which memories are vulnerable to a candidate update direction $g$, consider a virtual step of size $\eta>0$ and define
\begin{equation}
\delta_i=\ell(z_i;\theta_t-\eta g)-\ell(z_i;\theta_t).
\label{eq:loss-increment}
\end{equation}
A positive increment means that the update increases the loss on a stored example under its stored supervision. Ranking examples by this increment alone ignores how similar examples respond. It can therefore favor one isolated large increase over several moderate increases shared by neighboring memories. To account for this, we examine neighboring memories' loss changes under the same update. When nearby memories tend to respond similarly, their losses rising together provides stronger support for extra replay. If neighboring losses barely change or decrease, that support is weaker. We therefore assess each loss increase together with its neighbors before deciding how much replay it should receive.

Let $h_i=\phi(x_i)$ be an input-derived representation and let $d(h_i,h_j)$ measure separation in that representation space. Throughout, $d$ is finite, nonnegative, symmetric, and satisfies $d(h,h)=0$; the analysis does not require a triangle inequality. At a fixed scoring state and candidate direction, we model the signed increments as
\begin{equation}
\delta_i(g)=F_t(h_i;g)+\xi_i.
\label{eq:field-model}
\end{equation}
Here $F_i=F_t(h_i;g)$ is the conditional mean response at memory $i$, $F=(F_1,\ldots,F_m)$ collects these responses, and $\xi_i$ is the residual variation around that mean. Write $\mathbb E_t$ and $\mathbb P_t$ for expectation and probability conditional on the fixed scoring state and geometry, denoted by $\mathcal G_t$. Appendix~\ref{app:formal-conditioning} specifies this conditioning and gives a sufficient condition for the following spatial assumption.

\begin{assumption}[Spatial coherence]
\label{ass:spatial-coherence}
For some $L_F\geq0$, the conditional mean responses satisfy $|F_i-F_j|\leq L_Fd(h_i,h_j)$ for all memories $i,j$.
\end{assumption}
\begin{assumption}[Idiosyncratic score noise]
\label{ass:idiosyncratic-noise}
Conditional on $\mathcal G_t$, the residuals are independent, centered, and $\sigma$-sub-Gaussian: $\mathbb E_t[e^{a\xi_i}]\leq e^{\sigma^2a^2/2}$ for every $a\in\mathbb R$, with $\sigma>0$.
\end{assumption}
These assumptions express when neighboring scores provide useful corroboration: proximity limits variation in the mean response, and averaging can reduce residual variation that is not shared across memories. We then introduce the two kinds of forgetting as shown in Figure \ref{fig:neighborhood}.

\paragraph{Unsupported high-score spike.}
Suppose $\delta_i(g)$ is large and positive, while nearby memories have small or negative increments. Assumption~\ref{ass:spatial-coherence} limits how much their mean loss changes can differ. Any observed difference beyond this bound must come from residual variation. The large increment at $z_i$ therefore has little support from its neighbors that the whole region needs additional replay.

\paragraph{Spatially coherent high-risk region.}
Suppose instead that several nearby memories have positive increments of similar size. Their agreement supports a shared increase in the mean loss response. Under Assumption~\ref{ass:idiosyncratic-noise}, independent, zero-mean residuals become less likely to produce a large positive average as more memories contribute. Consistent increases across neighbors therefore provide stronger evidence for allocating additional replay to the region.

\section{SPHERE: Spatially coherent risk control for replay}
\label{sec:approach}
In this section, we introduce SPHERE, a general method for replay allocation compatible with different base models and training pipelines. It combines neighborhood aggregation with transport to compute replay weights, which are blended with uniform replay. SPHERE supports various settings, such as supervised learning, LLM instruction tuning, and reinforcement learning through their corresponding training losses.
\begin{figure}[h]
    \centering
    \includegraphics[width=1.03\linewidth]{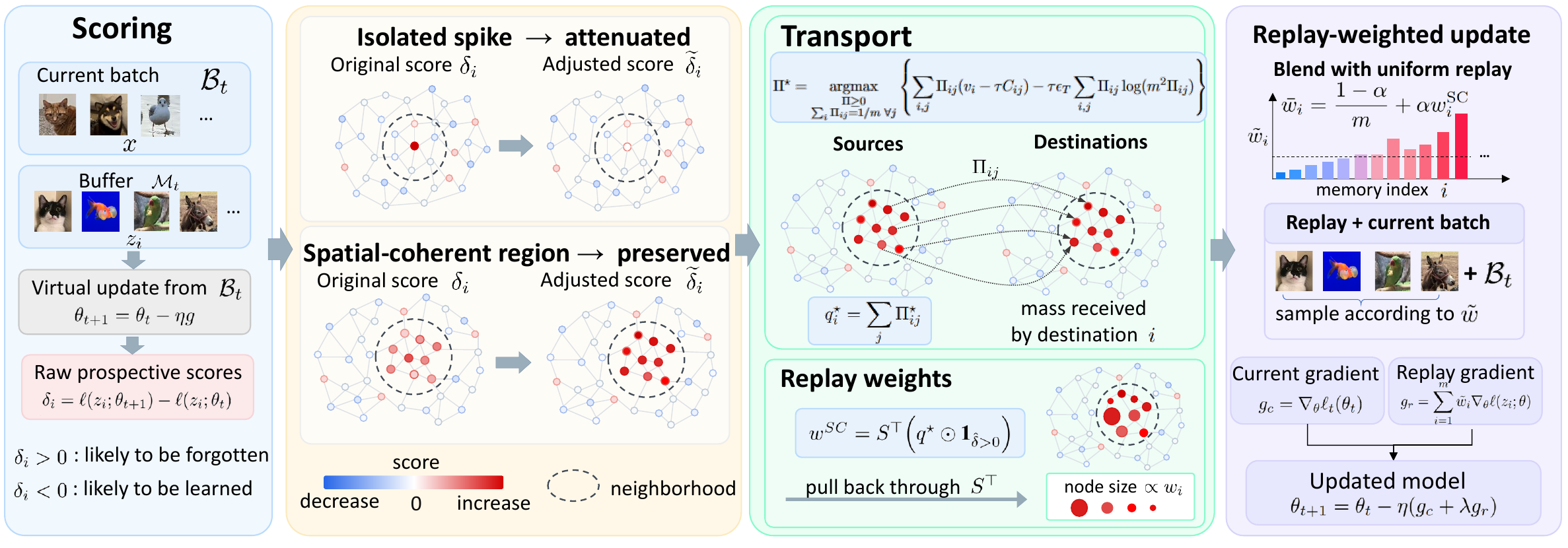}
    \caption{Overview of SPHERE.}
    \label{fig:pipeline}
\end{figure}
\subsection{Consensus corroborates prospective risk}
\label{subsec:geometry-estimation}

We use a representation kernel to determine how strongly neighboring memories contribute to each risk estimate. We score memories using the current-batch direction $g_c=\nabla_\theta L_t(\theta_t)$ and write $\delta_i=\delta_i(g_c)$ and $\delta=(\delta_1,\ldots,\delta_m)^\top$. Holding the representation geometry fixed during scoring, we define
\begin{equation}
    K_{ij}
    =
    \exp\!\left(-\frac{C_{ij}}{\epsilon_s}\right),
    \qquad
    S_{ij}
    =
    \frac{K_{ij}}{\sum_{k=1}^{m}K_{ik}},
    \qquad
    C_{ij}=d(h_i,h_j),
\end{equation}
where $\epsilon_s>0$ controls the spatial scale of aggregation. Nearby memories receive larger weights, while distant responses have less influence. Kernel weights depend on representation distance, whereas the learning objective enters through the loss changes. This separation provides a common aggregation rule across learning settings.

The row-stochastic matrix $S$ combines signed responses before retaining their positive part  $\widetilde{\delta}_i =  \sum_{j=1}^{m}S_{ij}\delta_j$, where $\widetilde{\delta}=S\delta$ and $\widetilde{s}=[\widetilde{\delta}]_+$. Thus, each risk estimate incorporates both the individual response and the evidence supplied by nearby memories. The kernel determines their relative influence. Let $n_i=\sum_{j\ne i}K_{ij}$ denote the neighboring kernel mass and $a_i=\sum_{j\ne i}K_{ij}\delta_j/n_i$ its weighted response when $n_i>0$. Since $K_{ii}=1$,
\begin{equation}
    \widetilde{\delta}_i
    =
    \frac{\delta_i+n_i a_i}{1+n_i}
    =
    a_i+\frac{\delta_i-a_i}{1+n_i}.
\end{equation}
For $n_i=0$, set $\widetilde{\delta}_i=\delta_i$. Small neighboring mass preserves more of the individual response, whereas larger mass contracts its deviation from neighboring evidence. If $\delta_i>0$ and $a_i\leq 0$, then $\widetilde{s}_i\leq\delta_i/(1+n_i)$; if $a_i=\delta_i>0$, the increase is preserved exactly.

The resulting estimator is useful when the reduction in residual variation outweighs the bias introduced by averaging responses across the neighborhood. Let $r_i=\sum_jS_{ij}d(h_i,h_j)$ be the local geometric radius. The following result makes the variance reduction and its geometric cost explicit.
\begin{proposition}[Geometry-controlled risk estimation]
\label{prop:kernel_estimator}
Under Assumptions~\ref{ass:spatial-coherence}--\ref{ass:idiosyncratic-noise},
\begin{equation}
\begin{aligned}
\big|\mathbb E_t[\widetilde\delta_i]-F_i\big|&\leq L_Fr_i,\\
\mathbb E_t[(\widetilde\delta_i-F_i)^2]&\leq L_F^2r_i^2+\sigma^2\sum_jS_{ij}^2,\\
\mathbb E_t[([\widetilde\delta_i]_+-[F_i]_+)^2]&\leq L_F^2r_i^2+\sigma^2\sum_jS_{ij}^2.
\end{aligned}
\label{eq:kernel-estimator-bounds}
\end{equation}
\end{proposition}
The kernel controls both the geometric radius $r_i$ and the effective neighborhood size $N_i^{\mathrm{eff}}=\frac{1}{\sum_j S_{ij}^2}=\frac{(\sum_j K_{ij})^2}{\sum_j K_{ij}^2}$.
The estimation bound therefore takes the form
$L_F^2r_i^2+\sigma^2/N_i^{\mathrm{eff}}$.
Distributing kernel weight across more memories reduces the
residual term, but averaging across dissimilar responses can
increase the geometric bias. More neighbors help only if their latent responses remain sufficiently similar. Specifically, with $b_i=\sum_jS_{ij}(F_j-F_i)$ and $v_j=\mathbb E_t[\xi_j^2]$, consensus has smaller mean-squared error than the raw score exactly when $b_i^2<v_i-\sum_jS_{ij}^2v_j$. Thus, corroboration is useful when variance reduction exceeds smoothing bias, not merely when the buffer is large. Appendix~\ref{app:kernel-proofs} gives the proof and exact error decomposition. The same kernel also determines how the transport objective's sensitivity is assigned back to the memories contributing to each risk estimate, which we will discuss in next section.

\subsection{Transport maps supported risk to replay coefficients}
\label{subsec:transport-allocation}
\label{subsec:host-integration}
Signed consensus reduces isolated score deviations, but residual spikes can still dominate allocation based only on magnitude. We therefore redistribute uniform mass over the historical buffer with a distance penalty. A distant location receives more mass from a source than a nearby competitor only when its advantage in supported risk offsets the extra distance cost, limiting how residual outliers redirect allocation across the memory.

For input scores $v\in\mathbb R^m$, evaluated at $v=\widetilde s=[S\delta]_+$, each memory $j$ supplies mass $1/m$ and $\Pi_{ij}$ assigns mass to location $i$:
\begin{equation}
\mathcal V_{\mathrm{OT}}(v)
=\max_{\substack{\Pi\geq0\\\sum_i\Pi_{ij}=1/m\ \forall j}}
\left\{\sum_{i,j}\Pi_{ij}(v_i-\tau C_{ij})
-\tau\epsilon_T\sum_{i,j}\Pi_{ij}\log(m^2\Pi_{ij})\right\},
\label{eq:soft-transport-allocation}
\end{equation}
where $C_{ij}=d(h_i,h_j)$, $\tau>0$ scales the distance penalty, $\epsilon_T>0$ sets the relative entropy scale. The term \(v_i-\tau d(h_i,h_j)\) balances the neighborhood-supported risk \(v_i=[\sum_k S_{ik}\delta_k]_+\) at destination \(i\) against the cost of reallocating mass from historical memory \(j\). A distant outlier therefore receives more mass from a source than a nearby competitor only when its score advantage exceeds the additional distance cost, discouraging residual spikes from redirecting allocation across unrelated memories. Destination masses remain free to favor high-risk locations, while entropy smooths the allocation. The unique solution is
\begin{equation}
\Pi_{ij}^{*}(v)=\frac{1}{m}
\frac{\exp((v_i-\tau C_{ij})/(\tau\epsilon_T))}
{\sum_{k=1}^{m}\exp((v_k-\tau C_{kj})/(\tau\epsilon_T))}.
\label{eq:soft-transport-plan}
\end{equation}
Writing $q_i^*(v)=\sum_j\Pi_{ij}^*(v)$, we have $\nabla_v\mathcal V_{\mathrm{OT}}(v)=q^*(v)$. Differentiating through consensus and clipping gives the sensitivity to the raw increments,
\begin{equation}
w^{\mathrm{SC}}=S^\top\!\left(q^*(\widetilde s)\odot\mathbf1_{\{\widetilde\delta>0\}}\right).
\label{eq:transport-pullback}
\end{equation}
Here $\odot$ denotes elementwise multiplication. The $S^\top$
pullback assigns sensitivity at each positive aggregate to the memories that contributed to it. It follows from differentiating the neighborhood-based risk with respect to the original loss changes, rather than from an independently chosen weighting rule. Appendix~E specifies the boundary subgradient. Each memory receives sensitivity from the locations with positive consensus it contributes to; the boundary subgradient is specified in Appendix~\ref{app:transport-proofs}. Since $w^{\mathrm{SC}}\geq0$ and $\|w^{\mathrm{SC}}\|_1\leq1$, we preserve baseline rehearsal by blending with uniform replay for $0\leq\alpha<1$:
\begin{equation}
\bar w_j=\frac{1-\alpha}{m}+\alpha w_j^{\mathrm{SC}},\qquad
s_w=\sum_j\bar w_j\in[1-\alpha,1].
\label{eq:blended-weight-mass}
\end{equation}
With these coefficients held fixed, the uniform-ER host uses
\begin{equation}
\theta_{t+1}=\theta_t-\eta\left[g_c+\lambda\sum_{j=1}^{m}\bar w_j\nabla_\theta\ell(z_j;\theta_t)\right].
\label{eq:geometry-weighted-replay}
\end{equation}
Appendix~\ref{app:core-implementation} gives unbiased sampling and host interfaces. And the transport value inherits the estimation bound at fixed scoring geometry. For each different task, it is flexible to choose the corresponding loss function $\ell$.

\begin{theorem}[Noise control for the transport value]
\label{thm:pipeline-noise}
Under Assumptions \ref{ass:spatial-coherence}--\ref{ass:idiosyncratic-noise},
\begin{equation}
\mathbb E_t[\mathcal V_{\mathrm{OT}}(\widetilde s)]\leq\mathcal V_{\mathrm{OT}}([F]_+)+L_F\max_i r_i+\sigma\sqrt{2\log(2m)\max_i\sum_jS_{ij}^2}.
\label{eq:pipeline-noise-control}
\end{equation}
\end{theorem}
The bound controls inflation of the scalar transport value through geometric bias and residual noise; systematic corruption can remain in $F$. Appendix~\ref{app:pipeline-proof} gives the proof, and Appendix~\ref{app:risk-descent-proof} establishes local descent for virtual-parameter replay under its stated conditions.

\section{Experiments}

\subsection{Continual language model post-training}
\label{sec:lnt}

We first evaluate whether SPHERE preserves earlier knowledge during continual language model post-training. The model sequentially learns text classification tasks covering sentiment analysis, topic classification, and natural language inference, while replaying stored examples to retain knowledge from earlier tasks. We use the 15 classification tasks of O-LoRA \citep{wang2023orthogonal}, expressed as instructions with verbalized labels and evaluated by exact match, without adding synthetic label corruption. T5-Large is adapted with LoRA of rank eight in a single pass of 860 updates with stream batches of 16.

\begin{table}[!ht]
\centering
\setlength{\abovecaptionskip}{2pt}
\setlength{\belowcaptionskip}{3pt}
\caption{Replay performance with limited memory and replay.
Values are mean$\pm$SD except Min.}
\label{tab:lnt-main}
\scriptsize
\setlength{\tabcolsep}{2.2pt}
\renewcommand{\arraystretch}{1.05}
\begin{tabular*}{\linewidth}{@{\extracolsep{\fill}}lcccccccc@{}}
\toprule
& \multicolumn{4}{c}{Buffer size 300}
& \multicolumn{4}{c}{Buffer size 3000} \\
\cmidrule(lr){2-5}\cmidrule(lr){6-9}
Method
& FP $\uparrow$ & Worst Forget $\downarrow$ & Tail Forget $\downarrow$ & Min. $\uparrow$
& FP $\uparrow$ & Worst Forget $\downarrow$ & Tail Forget  $\downarrow$ & Min. $\uparrow$ \\
\midrule
No replay
& .415$\pm$.146 & .811$\pm$.156 & .693$\pm$.173 & .130
& .415$\pm$.146 & .811$\pm$.156 & .693$\pm$.173 & .130 \\
ER
& .661$\pm$.075 & .245$\pm$.149 & .200$\pm$.137 & .423
& .641$\pm$.071 & .329$\pm$.143 & .269$\pm$.112 & .479 \\
Loss-prop
& .689$\pm$.045 & .188$\pm$.061 & \textbf{.142$\pm$.043} & .538
& .618$\pm$.050 & .325$\pm$.085 & .269$\pm$.076 & .502 \\
MIR-prop
& .675$\pm$.045 & .188$\pm$.083 & .152$\pm$.072 & .570
& .635$\pm$.042 & .296$\pm$.111 & .237$\pm$.079 & .561 \\
MIR-top
& .621$\pm$.023 & .244$\pm$.055 & .211$\pm$.060 & .588
& .633$\pm$.066 & .375$\pm$.205 & .285$\pm$.155 & .500 \\
SPHERE
& \textbf{.690$\pm$.021} & \textbf{.180$\pm$.053}
& .143$\pm$.035 & \textbf{.661}
& \textbf{.658$\pm$.028} & \textbf{.291$\pm$.093}
& \textbf{.228$\pm$.067} & \textbf{.619} \\
\bottomrule
\end{tabular*}
\end{table}

We compare SPHERE with no replay, ER, Loss-prop, MIR-prop, and MIR-top using buffers of 300 and 3000 and replay batches of 4. Replay methods use a 512-sample candidate pool refreshed from the buffer every four replay events. ER samples uniformly, Loss-prop samples proportionally to loss, and MIR-prop samples proportionally to positive one-step interference $[\delta_i]_+$.
MIR-top selects the top 16 interference scores, following the original MIR rule \citep{aljundi2019online}. Across both buffer sizes, Table~\ref{tab:lnt-main} shows that SPHERE performs higher FP, less forgetting  and lower observed variation than ER and the MIR variants. The results support using neighboring loss responses to guide replay allocation when only a small portion of past data can be rehearsed.

We next evaluate the neighborhood construction as a whole.
Setting $S=I$ removes signed aggregation and its corresponding $S^\top$ pullback while retaining transport with the original distance cost. In Table~\ref{tab:sphere-neighborhood-ablation}, full SPHERE improves final accuracy
by 4.42 percentage points and reduces worst forgetting by
7.90 points, while also improving new-task accuracy.
These gains support the joint use of neighborhood risk estimates
and their sensitivity-based replay coefficients.

\par\vspace{2pt}
\noindent\begin{minipage}[t]{0.405\linewidth}
\makeatletter\def\@captype{table}\makeatother
\scriptsize
\setlength{\abovecaptionskip}{0pt}
\setlength{\belowcaptionskip}{2pt}
\caption{Matched neighborhood-block ablation (\%).}
\label{tab:sphere-neighborhood-ablation}
\centering
\setlength{\tabcolsep}{1.8pt}
\renewcommand{\arraystretch}{1.03}
\begin{tabular*}{\linewidth}{@{\extracolsep{\fill}}lrrr@{}}
\toprule
Metric & $S=I$ & SPHERE & Gain \\
\midrule
Final acc. $\uparrow$ & 64.90 & \textbf{69.32} & +4.42 \\
Incremental acc. $\uparrow$ & 66.08 & \textbf{68.54} & +2.47 \\
New task acc. $\uparrow$ & 70.34 & \textbf{72.43} & +2.08 \\
Mean forgetting $\downarrow$ & 8.88 & \textbf{6.18} & +2.70 \\
Worst forgetting $\downarrow$ & 24.75 & \textbf{16.85} & +7.90 \\
Tail forgetting $\downarrow$ & 20.59 & \textbf{13.54} & +7.05 \\
\bottomrule
\end{tabular*}
\end{minipage}\hfill
\begin{minipage}[t]{0.575\linewidth}
\makeatletter\def\@captype{table}\makeatother
\scriptsize
\setlength{\abovecaptionskip}{0pt}
\setlength{\belowcaptionskip}{2pt}
\caption{Component controls with scarce memory and replay.
Values are mean$\pm$SD except Min.}
\label{tab:lnt-ablation}
\centering
\setlength{\tabcolsep}{1.3pt}
\renewcommand{\arraystretch}{1.03}
\begin{tabular*}{\linewidth}{@{\extracolsep{\fill}}lcccc@{}}
\toprule
Variant & FP $\uparrow$ & Worst $\downarrow$
& Tail $\downarrow$ & Min. $\uparrow$ \\
\midrule
Full SPHERE
& \textbf{.690$\pm$.021} & \textbf{.180$\pm$.053}
& \textbf{.143$\pm$.035} & \textbf{.661} \\
Consensus only
& .680$\pm$.021 & .189$\pm$.054 & .154$\pm$.030 & .647 \\
No transport
& .640$\pm$.063 & .275$\pm$.129 & .227$\pm$.110 & .478 \\
Permuted geometry
& .666$\pm$.056 & .213$\pm$.082 & .188$\pm$.078 & .497 \\
Matched MIR
& .682$\pm$.048 & .192$\pm$.113 & .161$\pm$.087 & .517 \\
No uniform floor
& .673$\pm$.034 & .230$\pm$.082 & .183$\pm$.063 & .608 \\
\bottomrule
\end{tabular*}
\end{minipage}
\par\vspace{2pt}
Table \ref{tab:lnt-ablation} examines allocation after neighborhood aggregation. Consensus only replaces transport with an ESS-matched softmax of the same supported scores, retaining the positive-response mask, $S^\top$ pullback, and uniform mixture. Full SPHERE achieves 0.690 final performance versus 0.680, with lower worst and tail forgetting. The lower reported mean with permuted geometry further supports the importance of alignment between memories and their geometry. All instruction-tuning controls use unit-mass replay, so these gains do not require explicit replay-strength scaling.

\subsection{Noisy-label continual learning}

\label{sec:cifar}

Corrupted labels make an individual loss increase ambiguous, since rehearsing an incorrect target can reduce accuracy on clean test data. We next test the performace of SPHERE on conventional classfication task to test whether neighborhood evidence improves performance across different training objectives. We integrate SPHERE with ER, ER-ACE~\citep{caccia2022new}, and X-DER~\citep{boschini2022class}. Each integration retains its host's losses and memory policy. MIR-prop~\citep{aljundi2019online} and DER++~\citep{buzzega2020dark} provide additional replay baselines. We also evaluate SPHERE within PuriDivER's purity estimation and sample-partition rules~\citep{bang2022online}. PuriDivER, GDumb~\citep{GDumb}, and ABS-based~\citep{millunzi2024may} use additional memory fitting and are reported separately. Split CIFAR-100 and TinyImageNet each contain ten tasks, with memory capacities of 2,000 and 4,000 examples, respectively. Online training uses a single pass with 0, 20, 40, or 60\% symmetric label noise. The main evaluation measures classification accuracy on clean test data without task identity. Experiment details and additional results on MNIST and CIFAR-10 appear in Appendix \ref{app:experiments}.
\begin{table}[h]
\centering
\caption{Replay across hosts under symmetric label noise.}
\label{tab:cifar-main}\label{tab:c100-main}\label{tab:tiny-main}
\scriptsize
\setlength{\tabcolsep}{1.5pt}
\renewcommand{\arraystretch}{1.07}
\begin{tabular*}{\linewidth}{@{\extracolsep{\fill}}lcccccccc@{}}
\toprule
& \multicolumn{4}{c}{CIFAR-100 ($K=2{,}000$)} & \multicolumn{4}{c}{TinyImageNet ($K=4{,}000$)}\\
\cmidrule(lr){2-5}\cmidrule(l){6-9}
Method & Clean & 20\% & 40\% & 60\% & Clean & 20\% & 40\% & 60\%\\
\midrule
ER & 19.6$\pm$3.4 & 14.3$\pm$1.1 & 10.2$\pm$1.1 & 5.1$\pm$0.8 & 16.3$\pm$1.2 & 12.0$\pm$0.6 & 7.4$\pm$0.6 & 3.3$\pm$0.4\\
ER + SPHERE & \textbf{20.9$\pm$1.8} & 15.7$\pm$1.2 & 12.2$\pm$0.8 & 8.0$\pm$0.7 & 15.3$\pm$1.4 & 13.5$\pm$0.9 & 10.1$\pm$0.6 & 6.1$\pm$0.7\\
MIR-prop & 19.8$\pm$1.4 & 15.0$\pm$2.3 & 10.2$\pm$1.6 & 6.2$\pm$0.8 & 17.1$\pm$0.5 & 11.6$\pm$1.1 & 8.0$\pm$0.8 & 4.7$\pm$0.2\\
DER++ & 15.3$\pm$4.1 & 10.9$\pm$4.1 & 9.0$\pm$2.0 & 5.6$\pm$1.0 & 9.8$\pm$1.0 & 7.4$\pm$0.9 & 4.8$\pm$0.7 & 3.4$\pm$0.3\\
ER-ACE & 20.4$\pm$3.4 & 16.8$\pm$2.1 & 12.9$\pm$1.5 & 8.1$\pm$1.1 & 17.9$\pm$1.7 & 14.4$\pm$0.8 & 9.6$\pm$0.8 & 5.5$\pm$0.4\\
ER-ACE + SPHERE & 18.7$\pm$2.9 & 16.6$\pm$1.6 & 13.5$\pm$1.6 & 10.4$\pm$0.6 & 16.6$\pm$1.4 & 14.2$\pm$0.6 & 11.0$\pm$1.1 & 6.8$\pm$0.6\\
X-DER & 20.2$\pm$5.0 & 16.2$\pm$6.1 & 14.1$\pm$1.8 & 8.3$\pm$2.4 & 17.9$\pm$1.6 & 17.5$\pm$1.0 & 10.4$\pm$1.2 & 6.2$\pm$1.6\\
X-DER + SPHERE & 20.7$\pm$4.6 & \textbf{19.6$\pm$3.0} & \textbf{16.1$\pm$1.6} & \textbf{10.9$\pm$1.1} & \textbf{19.4$\pm$1.1} & \textbf{18.2$\pm$0.8} & \textbf{12.8$\pm$1.1} & \textbf{8.2$\pm$0.4}\\
\midrule
\multicolumn{9}{l}{\emph{Separate buffer training}}\\
GDumb & \textbf{25.2$\pm$1.0} & 17.7$\pm$0.9 & 12.5$\pm$0.6 & 7.4$\pm$0.5 & \textbf{19.4$\pm$0.3} & 13.0$\pm$0.5 & 7.7$\pm$0.7 & 4.3$\pm$0.2\\
ABS-based & -- & 32.9$\pm$0.7 & 32.6$\pm$0.8 & 24.9$\pm$0.9 & -- & 26.8$\pm$0.8 & 23.4$\pm$0.8 & 16.2$\pm$0.3\\
PuriDivER & -- & 38.6$\pm$0.7 & 37.2$\pm$1.5 & 32.2$\pm$0.6 & -- & 30.1$\pm$0.3 & 28.7$\pm$0.3 & 24.6$\pm$0.2\\
PuriDivER + SPHERE & -- & \textbf{39.2$\pm$2.0} & \textbf{39.1$\pm$1.1} & \textbf{36.2$\pm$1.5} & -- & \textbf{30.8$\pm$0.1} & \textbf{31.9$\pm$1.4} & \textbf{27.4$\pm$1.3}\\
\bottomrule
\end{tabular*}
\end{table}

Table~\ref{tab:cifar-main} shows that SPHERE improves all three replay hosts on both datasets under moderate and high label noise. The gains also extend to PuriDivER, which uses purity estimates to partition stored examples and guide training under corrupted supervision. The SPHERE integrations improve the evaluated replay hosts under moderate and high label noise. The improvements across hosts support its broader use in replay allocation, while gains under cleaner supervision depend on the host. We next examine whether the gains extend beyond final classification accuracy. Class-IL and Task-IL measure accuracy without and with task identity, respectively, while AAA summarizes accuracy throughout training. 
\begin{wraptable}{r}{0.69\linewidth}
\centering
\vspace{-5pt}
\caption{CIFAR-10 at 20\% noise, $K=500$ (mean$\pm$SD).}
\label{tab:vision-retention-main}
\scriptsize
\setlength{\tabcolsep}{2.0pt}
\renewcommand{\arraystretch}{1.10}
\begin{tabular}{lcccccc}\toprule
Method & Class-IL$\uparrow$ & Task-IL$\uparrow$ & AAA$\uparrow$ & avgF$\downarrow$ & worst-class$\uparrow$ & CVaR$_{0.2}\downarrow$ \\ \midrule
ER & 33.4$\pm$4.5 & 76.1$\pm$5.1 & 51.4$\pm$4.8 & 42.9$\pm$5.6 & 5.6$\pm$4.5 & 72.5$\pm$5.7 \\
ER-ACE & 36.1$\pm$4.5 & 77.9$\pm$6.5 & 53.2$\pm$2.1 & 25.3$\pm$5.5 & 11.1$\pm$7.2 & 54.8$\pm$8.9 \\
GDumb & 36.8$\pm$0.9 & 77.0$\pm$2.8 & 28.1$\pm$0.7 & 20.2$\pm$7.4 & 8.4$\pm$5.4 & 51.2$\pm$19.0 \\
DER++ & 28.2$\pm$4.3 & 80.7$\pm$4.8 & 48.4$\pm$4.7 & 57.5$\pm$9.3 & 1.3$\pm$1.7 & 81.3$\pm$7.0 \\
MIR & 32.1$\pm$4.4 & 75.6$\pm$4.1 & 48.5$\pm$3.5 & 36.9$\pm$7.0 & 5.5$\pm$3.6 & 71.1$\pm$6.9 \\
X-DER & 28.1$\pm$8.2 & 75.1$\pm$8.0 & 44.9$\pm$5.3 & 25.0$\pm$9.9 & 0.0$\pm$0.0 & 60.0$\pm$16.1 \\
\textbf{ER-ACE + SPHERE} & \textbf{38.8$\pm$4.2} & \textbf{81.0$\pm$4.5} & \textbf{54.6$\pm$3.1} & \textbf{15.3$\pm$6.7} & \textbf{19.2$\pm$2.4} & \textbf{46.5$\pm$10.8} \\
\bottomrule\end{tabular}\vspace{-5pt}
\end{wraptable}
We also
report average forgetting (avgF), worst-class accuracy, and tail forgetting
($\mathrm{CVaR}_{0.2}$) to assess retention beyond the average prediction.
Table~\ref{tab:vision-retention-main} shows that adding SPHERE to ER-ACE
improves every reported metric. The gains in worst-class accuracy and tail
forgetting support the use of neighborhood information to protect knowledge
that is most affected by subsequent updates.

\begin{wraptable}{r}{0.48\linewidth}\vspace{-20pt}
\centering
\caption{Neighborhood-block ablation on clean Split CIFAR-10.}
\label{tab:cifar-neighborhood-ablation}
\small
\setlength{\tabcolsep}{4pt}
\renewcommand{\arraystretch}{1.12}
\begin{tabular}{@{}lcc@{}}
\toprule
Metric & SPHERE & $S=I$ \\
\midrule
Final acc $\uparrow$
 & {\bf 46.61 $\pm$ 2.19} & 45.42$\pm$2.99 \\
Stage-average acc $\uparrow$
 & {\bf 59.01 $\pm$ 3.00} & 58.93$\pm$3.30 \\
\midrule
Mean forgetting $\downarrow$
 & {\bf 35.10 $\pm$ 5.90} & 36.55$\pm$7.71 \\
Worst forgetting $\downarrow$
 & {\bf 62.66 $\pm$ 9.63} & 69.85$\pm$12.42 \\
Top-2 forgetting $\downarrow$
 & {\bf 59.53 $\pm$ 8.81} & 61.84$\pm$10.69  \\
\bottomrule
\end{tabular}
\end{wraptable}

Finally, we examine whether neighborhood information remains useful when
labels are clean. On clean Split CIFAR-10, we set $S=I$ and jointly remove
neighborhood averaging and its associated weight pullback while retaining
the other components. Table~\ref{tab:cifar-neighborhood-ablation} shows that
full SPHERE reduces mean and severe forgetting, alongside higher final
accuracy and similar stage-average accuracy. The results support neighborhood
information as a means of preserving earlier knowledge even without label
corruption.

\subsection{Code generation with incomplete test rewards}

Beyond classification, we further implement whether SPHERE can preserve earlier skills during continual reinforcement learning for code generation. For a problem $x$, the policy $\pi_\theta$ generates a program $y$ token by token and receives the terminal reward
$r^{\mathrm{vis}}(x,y)=\mathbf{1}\{y\text{ passes the visible test for }x\}$.
A rewarded program can still fail other tests, so replay must allocate
limited capacity using loss changes derived from incomplete feedback.
The expected reward and its policy gradient are
\begin{equation}
\begin{aligned}
J_t(\theta)
&=\mathbb{E}_{x\sim\mathcal D_t,\,
                 y\sim\pi_\theta(\cdot\mid x)}
  [r^{\mathrm{vis}}(x,y)],\\
\nabla_\theta J_t(\theta)
&=\mathbb{E}\!\left[
  \bigl(r^{\mathrm{vis}}(x,y)-b_t(x)\bigr)
  \sum_{u=1}^{|y|}
  \nabla_\theta\log\pi_\theta(y_u\mid x,y_{<u})
  \right],
\end{aligned}
\label{eq:code-policy-gradient}
\end{equation}
where $\mathcal D_t$ is the current problem distribution and $b_t(x)$
is a problem-dependent baseline.

We construct four successive stages from KodCode-V1
\citep{xu2025kodcode}, covering Package, Data Structure/Algorithm,
Leetcode, and Codeforces/Apps. After removing near duplicates, each
stage contains 120 training problems and 100 disjoint evaluation
problems. We adapt Qwen2.5-Coder-1.5B-Instruct \citep{hui2024qwen2}
with rank-16 LoRA. Two passes over each stage give 30 updates per
stage and 120 in total, with eight current problems per update.
Evaluation uses greedy decoding and requires passing the full test
suite, which is unavailable to training and replay selection.

We optimize the policy with a GRPO-style clipped loss
\citep{shao2024deepseekmath,liu2025understanding}. For each problem
$x_i$, we sample a group $\mathcal G_i$ of eight programs and compute
$A_{ik}=r_{ik}-\frac{1}{8}\sum_{\ell=1}^{8}r_{i\ell}$, where
$r_{ik}=r^{\mathrm{vis}}(x_i,y_{ik})$. The token-level surrogate
weights policy probability ratios by these centered advantages, with
clipping to limit the incentive for large probability changes.
Let $\mathcal L_i^{\mathrm{GRPO}}(\theta;\mathcal G_i)$ denote the
resulting group loss. Only groups containing both reward outcomes
are scored, since all advantages vanish otherwise. SPHERE assesses whether an update on current problems would increase
the policy loss on an earlier problem. We take a AdamW step to obtain $\theta_t^{\mathrm v}$ and compute $\delta_i=
\mathcal L_i^{\mathrm{GRPO}}(\theta_t^{\mathrm v};\mathcal G_i)
-\mathcal L_i^{\mathrm{GRPO}}(\theta_t;\mathcal G_i)$, holding the sampled programs and advantages fixed across both evaluations. SPHERE combines these changes across nearby prompt embeddings before allocating replay weights. All replay methods store up to 480 problems, assess 48 candidates per update, and replay eight groups with unit total weight. Hyperparameters are selected on disjoint development data and seeds.

\begin{wraptable}{r}{0.59\linewidth}
\centering
\vspace{-15pt}
\caption{Code-generation results. }
\label{tab:code-main}
\renewcommand{\arraystretch}{1.02}
\scriptsize
\begin{tabular}{@{\extracolsep{\fill}}lrrrrrr@{}}
\toprule
Method & Final $\uparrow$ & Old $\uparrow$ & Worst $\uparrow$
& $F\downarrow$ & BWT $\uparrow$  \\
\midrule
Uniform replay & 42.75$\pm$2.92 & 43.56 & 40.00 & 4.00 & -2.78 \\
MIR & 44.50$\pm$1.63 & 45.78 & 38.67 & 4.56 & -3.89  \\
No Transport & 42.67$\pm$1.70 & 41.78 & 39.67 & 6.56 & -5.56  \\
Permuted geometry & 48.58$\pm$1.88 & 45.78 & 41.67 & 3.56 & -2.89 &  \\
SPHERE & \textbf{49.08$\pm$0.80} & \textbf{46.89} & \textbf{43.67}
& \textbf{3.33} & \textbf{-1.56}  \\
\bottomrule
\end{tabular}\vspace{-10pt}
\end{wraptable}

Table~\ref{tab:code-main} shows SPHERE performs stronger final accuracy and retention  than for uniform replay and interference-based retrieval, alongside positive acquisition gains. The results suggest that the replay controller can help balance retention and adaptation when policy updates rely on incomplete rewards.
Figure~\ref{fig:code-training} complements the final comparison by tracking visible-test reward during training and full-suite accuracy after each stage. Together, the panels relate progress under the training reward to performance under broader correctness checks as the task distribution changes. 

\begin{figure}[h]
    \centering
    \includegraphics[width=0.95\linewidth]{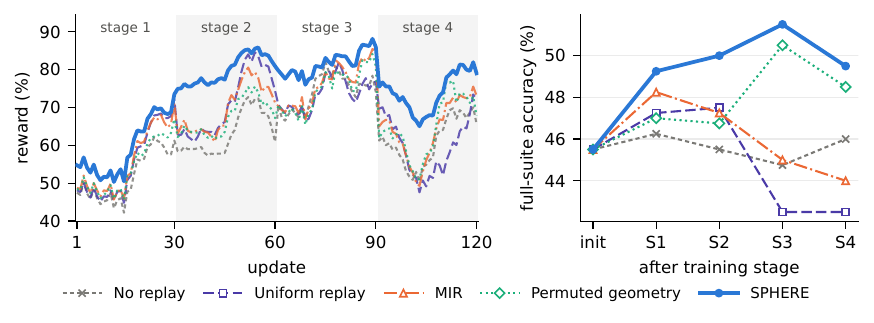}\vspace{-10pt}
    \caption{Reward during training (left) and evaluation accuracy after each stage (right).}
    \label{fig:code-training}\vspace{-10pt}
\end{figure}

\section{Conclusion}
We introduced SPHERE, a general replay-allocation method applicable across a broad range of learning settings. SPHERE combines kernel-based aggregation of signed loss changes with entropy-regularized transport, directing replay toward supported high-risk regions while penalizing long-distance mass transfers. Replay coefficients are derived from the transport objective's sensitivity to the original loss changes and blended with uniform replay to retain baseline rehearsal. Our analysis establishes when neighborhood aggregation improves risk estimation and bounds transport-value inflation due to residual noise and smoothing bias. Experiments demonstrate improved accuracy and retention across noisy-label vision tasks, continual language-model instruction tuning, and code-generation reinforcement learning with incomplete test rewards. However, computing pairwise distances and transport allocations adds overhead as the scored memory set grows, and our future work will explore sparse approximations to reduce this cost in large-scale continual training.

\bibliography{iclr2027_conference}
\bibliographystyle{iclr2027_conference} 
\clearpage
\appendix

\section{Core implementation interface and limits of the guarantees}
\label{app:core-implementation}
\paragraph{A fully specified ideal update.}
Algorithm~\ref{alg:core-controller} implements the mathematical controller on a fixed scored set. 

\begin{algorithm}[ht]
\caption{Core SPHERE on a fixed scored set}
\label{alg:core-controller}
\begin{algorithmic}[1]
\Require Model $\theta_t$, current batch $B_t$, scored memories $\{z_i\}_{i=1}^m$, input features $\{h_i\}$, parameters $\eta,\epsilon_s,\tau,\epsilon_T>0$, $\lambda\geq0$, $\alpha\in[0,1)$, replay batch size $b\geq1$.
\State Fix the scored set, input features, stored targets, and scoring randomness.
\State Compute $g_c=\nabla_\theta L_t(\theta_t)$ and virtual parameter $\theta^v=\theta_t-\eta g_c$.
\State Compute $\delta_i=\ell(z_i;\theta^v)-\ell(z_i;\theta_t)$ for each memory.
\State Form $C_{ij}=d(h_i,h_j)$, $K_{ij}=\exp(-C_{ij}/\epsilon_s)$, and row-normalized $S$.
\State Compute $\widetilde\delta=S\delta$ and $\widetilde s=[\widetilde\delta]_+$.
\State Compute $\Pi^*(\widetilde s)$ by Eq.~(\ref{eq:soft-transport-plan}) and $q_i^*=\sum_j\Pi^*_{ij}$.
\State Set $w^{\mathrm{SC}}=S^\top(q^*\odot\mathbf1_{\{\widetilde\delta>0\}})$.
\State Set $\bar w_i=(1-\alpha)/m+\alpha w_i^{\mathrm{SC}}$, $s_w=\sum_i\bar w_i$, and $p_i=\bar w_i/s_w$.
\State Draw $J_1,\ldots,J_b$ independently with replacement from $p$.
\State Form $\widehat g_r=(s_w/b)\sum_{r=1}^{b}\nabla_\theta\ell(z_{J_r};\theta_t)$, treating all weights as constants.
\State Update $\theta_{t+1}=\theta_t-\eta(g_c+\lambda\widehat g_r)$.
\end{algorithmic}
\end{algorithm}

Let $\mathcal H_t$ contain the current model, scored memories, scoring outcomes, weights, and per-example replay gradients, all fixed before drawing the replay indices. Conditional on $\mathcal H_t$, independence and $\mathbb P(J_r=i\mid\mathcal H_t)=p_i$ give
\begin{equation}
\mathbb E[\widehat g_r\mid\mathcal H_t]=s_w\sum_i p_i\nabla_\theta\ell(z_i;\theta_t)=\sum_i\bar w_i\nabla_\theta\ell(z_i;\theta_t).
\end{equation}
The same identity applies to conditional mean gradients if fresh per-example randomness is used independently and preserves those means.

Normalizing $\bar w$ and averaging sampled losses without the factor $s_w$ instead realizes unit-mass replay. If unequal-probability sampling is performed without replacement, its inclusion probabilities need not equal $bp_i$; an uncorrected minibatch mean is not generally unbiased for the displayed weighted sum. With positive inclusion probabilities $\pi_i^{\mathrm{inc}}$, the inclusion-probability-corrected form $\sum_{i\in\mathcal B}(\bar w_i/\pi_i^{\mathrm{inc}})\nabla\ell_i$ is unbiased, but computing or using those inclusion probabilities is a separate implementation choice. This correction is distinct from the uncorrected without-replacement retrieval used in the reported configuration descriptions.

\section{Conditional field model and sufficient gradient regularity}
\label{app:formal-conditioning}

This appendix makes explicit what the conditional field represents and when nearby representations can have similar expected responses. Fix a scoring state $\mathcal G_t$ containing $(\theta_t,g)$, the scored candidate set, the input representations, and any randomness used to construct the geometry. The field values $F_i$ and kernel $S$ are $\mathcal G_t$-measurable, and Assumption~\ref{ass:idiosyncratic-noise} is imposed conditional on this state. Consequently $F_i=\mathbb E_t[\delta_i(g)]$. The assumption concerns residual variation remaining under the conditional model; simply freezing a trained model does not establish it. In particular, historical label noise may become dependent after conditioning on a model trained with those labels.

For a sufficient construction, take a generic example $Z=(X,\widetilde Y)$ under a fixed family of conditional laws given $\phi(X)=h$ and $\mathcal G_t$. Choose versions of these laws and their conditional expectations at the representations under consideration, and suppose the following increments are integrable:
\begin{equation}
F_t(h;g)=\mathbb E[\ell(Z;\theta_t-\eta g)-\ell(Z;\theta_t)\mid\phi(X)=h,\mathcal G_t].
\label{eq:conditional-mean-field}
\end{equation}
To use this construction for the buffer, assume that each memory's conditional mean increment equals $F_t(h_i;g)$. This identification is part of the working model. The next proposition supplies a sufficient condition for spatial coherence within that model: if nearby representations have similar expected gradients along the virtual-update path, they also have similar expected loss increments. It does not derive either condition from adaptive buffer construction.

\begin{proposition}[Gradient regularity implies spatial coherence]
\label{prop:gradient-implies-coherence}
Similar conditional gradients at nearby representations imply similar expected virtual-update increments. For $s\in[0,1]$, let $\theta_{t,s}=\theta_t-s\eta g$ and
\begin{equation}
G_t(h;\theta)
=
\mathbb{E}\!\left[
\nabla_\theta\ell(Z;\theta)
\mid \phi(X)=h,\mathcal G_t
\right].
\label{eq:conditional-gradient-field}
\end{equation}

Assume that $s\mapsto\ell(Z;\theta_t-s\eta g)$ is almost surely absolutely continuous and has derivative $-\eta\langle\nabla_\theta\ell(Z;\theta_{t,s}),g\rangle$ for almost every $s\in[0,1]$. Assume the gradient integrand is jointly measurable and that

\begin{equation}
\mathbb E\!\left[\int_0^1\|\nabla_\theta\ell(Z;\theta_{t,s})\|_2\,ds\mid\phi(X)=h,\mathcal G_t\right]<\infty
\end{equation}
for every representation under consideration. Suppose also that, for every $s\in[0,1]$,
\begin{equation}
\|G_t(h;\theta_{t,s})-G_t(h';\theta_{t,s})\|_2
\leq L_Gd(h,h').
\label{eq:conditional-gradient-lipschitz}
\end{equation}
Then
\begin{equation}
|F_t(h;g)-F_t(h';g)|
\leq
\eta L_G\|g\|_2d(h,h'),
\label{eq:derived-field-lipschitz}
\end{equation}
so Assumption~\ref{ass:spatial-coherence} holds with $L_F=\eta L_G\|g\|_2$.
\end{proposition}

\begin{proof}

For a generic example $Z$, absolute continuity and the assumed path derivative give

\begin{equation}
\ell(Z;\theta_t-\eta g)-\ell(Z;\theta_t)
=
-\eta\int_0^1
\langle\nabla_\theta\ell(Z;\theta_{t,s}),g\rangle ds.
\label{eq:ftc-virtual-increment}
\end{equation}
The integrability assumption permits conditional expectation and integration to be interchanged, yielding
\begin{equation}
F_t(h;g)
=
-\eta\int_0^1
\langle G_t(h;\theta_{t,s}),g\rangle ds.
\label{eq:conditional-field-integral}
\end{equation}
Therefore, by Cauchy--Schwarz and Eq.~(\ref{eq:conditional-gradient-lipschitz}),
\begin{equation}
\begin{aligned}
|F_t(h;g)-F_t(h';g)|
&\leq
\eta\int_0^1
\|G_t(h;\theta_{t,s})-G_t(h';\theta_{t,s})\|_2
\|g\|_2ds\\
&\leq
\eta L_G\|g\|_2d(h,h'),
\end{aligned}
\end{equation}
which proves Eq.~(\ref{eq:derived-field-lipschitz}).
\end{proof}

\subsection{Fixed-state symmetric corruption as an illustration}

Let the label space contain $C\geq2$ classes and let $\rho\in[0,1]$. With fixed inputs and clean labels $y_i$, symmetric corruption independently draws each observed label according to

\begin{equation}
\mathbb{P}(\widetilde Y_i=c\mid Y_i=y_i)
=
\begin{cases}
1-\rho, & c=y_i,\\
\rho/(C-1), & c\neq y_i.
\end{cases}
\label{eq:symmetric-label-noise}
\end{equation}
The clean label is introduced only to describe the data-generating process and is never available to SPHERE. For a fixed input, model state, and candidate direction, let
\begin{equation}
\delta_i(c;g)
=
\ell(x_i,c;\theta_t-\eta g)-\ell(x_i,c;\theta_t)
\label{eq:class-conditioned-increment}
\end{equation}
denote the increment that would be observed under label $c$. 

The corruption-averaged signed response and its centered residual are
\begin{equation}
\begin{aligned}
F_i^{(\rho)}(g)
&=
(1-\rho)\delta_i(y_i;g)
+
\frac{\rho}{C-1}\sum_{c\neq y_i}\delta_i(c;g),\\
\xi_i^{(\rho)}(g)
&=
\delta_i(\widetilde y_i;g)-F_i^{(\rho)}(g).
\end{aligned}
\label{eq:noisy-risk-decomposition}
\end{equation}

For this illustrative calculation, the inputs, clean labels, scoring model, candidate direction, and geometry are fixed independently of the hypothetical corrupted-label draws. Then $\xi_i^{(\rho)}$ are centered and independent. If all class-conditioned increments are finite, their finite ranges also imply sub-Gaussian residuals. This gives one possible residual model, but does not establish spatial coherence of $F^{(\rho)}$; that remains a separate requirement. In actual continual training, historical corruption can affect the model and its geometry, so neither conditional independence nor meaningful neighborhoods follows automatically from independent stream-label corruption.

For cross-entropy with positive class probabilities at both scoring states,
\begin{equation}
\delta_i(c;g)=\log\frac{p_{\theta_t}(c\mid x_i)}{p_{\theta_t-\eta g}(c\mid x_i)}.
\label{eq:ce-noisy-increment}
\end{equation}
An update that increases the true-class probability and decreases a corrupted-class probability therefore gives a negative clean-target increment and a positive corrupted-target increment. This is a concrete way that the same prediction change can improve clean prediction while appearing harmful under stored supervision. It establishes a possibility, not a rule that every corrupted label creates an isolated spike.

\section{Concentration of fixed-region and pairwise responses}
\label{app:spatial-evidence-proof}

Let $A$ be a nonempty $\mathcal G_t$-measurable region with $k$ examples, chosen without inspecting the residual realizations. Let $\bar\delta_A$ and $\bar F_A$ denote its average observed and latent increments. The same fixed-choice requirement applies to the pair $i,j$.

\begin{proposition}[Statistical spike--region separation]
\label{prop:spatial-support-concentration}
Suppose Assumptions~\ref{ass:spatial-coherence} and~\ref{ass:idiosyncratic-noise} hold. For distinct $i,j$ and every $x>0$,
\begin{equation}
\begin{aligned}
\mathbb{P}_t\!\left(
|\delta_i(g)-\delta_j(g)|
\geq L_Fd(h_i,h_j)+x
\right)
&\leq 2e^{-x^2/(4\sigma^2)},\\
\mathbb{P}_t\!\left(
|\bar\delta_A-\bar F_A|\geq x
\right)
&\leq 2e^{-kx^2/(2\sigma^2)}.
\end{aligned}
\label{eq:spike-region-concentration}
\end{equation}
\end{proposition}

These are concentration bounds, not a calibrated classifier of clean examples or an adaptive region-selection guarantee.

\begin{proof}
All probabilities and expectations below are conditional on $\mathcal{G}_t$. From Eq.~(\ref{eq:field-model}),
\begin{equation}
\delta_i(g)-\delta_j(g)=F_i-F_j+\xi_i-\xi_j.
\end{equation}
By Assumption~\ref{ass:spatial-coherence}, the event in the first line of Eq.~(\ref{eq:spike-region-concentration}) is contained in $\{|\xi_i-\xi_j|\geq x\}$. Conditional independence and Assumption~\ref{ass:idiosyncratic-noise} give
\begin{equation}
\begin{aligned}
\mathbb{E}_t[e^{u(\xi_i-\xi_j)}]
&=
\mathbb{E}_t[e^{u\xi_i}]\mathbb{E}_t[e^{-u\xi_j}]\\
&\leq e^{\sigma^2u^2}.
\end{aligned}
\end{equation}
Thus, $\xi_i-\xi_j$ is $\sqrt{2}\sigma$-sub-Gaussian, and its two-sided tail satisfies
\begin{equation}
\mathbb{P}_t(|\xi_i-\xi_j|\geq x)
\leq2e^{-x^2/(4\sigma^2)},
\end{equation}
which proves the first line of Eq.~(\ref{eq:spike-region-concentration}).

For a nonempty region $A$ of size $k$,
\begin{equation}
\bar\delta_A-\bar F_A
=
\frac{1}{k}\sum_{i\in A}\xi_i.
\end{equation}
For every $u\in\mathbb{R}$,
\begin{equation}
\begin{aligned}
\mathbb{E}_t\!\left[
\exp\!\left(\frac{u}{k}\sum_{i\in A}\xi_i\right)
\right]
&=
\prod_{i\in A}\mathbb{E}_t[e^{u\xi_i/k}]\\
&\leq
\exp\!\left(\frac{\sigma^2u^2}{2k}\right).
\end{aligned}
\end{equation}
Hence, the regional average is $\sigma/\sqrt{k}$-sub-Gaussian, and the standard two-sided tail bound proves the second line of Eq.~(\ref{eq:spike-region-concentration}).
\end{proof}

\section{Proofs for geometry-supported estimation}
\label{app:kernel-proofs}

\subsection{Exact spike--region contrast and smoothing order}

The ranking contrast can occur even for the dense positive kernel in the main text. Take two groups $B,A$ of $k\geq2$ memories each. All features in $B$ equal $h_B$, all features in $A$ equal $h_A$, and $d(h_A,h_B)=D>0$. Within-group kernel weights are one; cross-group weights are $\varepsilon=\exp(-D/\epsilon_s)\in(0,1)$. Let one memory $i\in B$ have increment $b$, let the remaining $B$ increments be zero, and let every $A$ increment be $a$, where $0<a<b<ka$.

All rows have denominator $k(1+\varepsilon)$. For a memory $c\in A$,
\begin{equation}
\begin{aligned}
\widetilde\delta_i&=\frac{b+\varepsilon ka}{k(1+\varepsilon)},\\
\widetilde\delta_c&=\frac{ka+\varepsilon b}{k(1+\varepsilon)},\\
\widetilde\delta_c-\widetilde\delta_i&=\frac{(1-\varepsilon)(ka-b)}{k(1+\varepsilon)}>0.
\end{aligned}
\end{equation}
The raw spike has $b>a$, but the coherent group has larger smoothed scores. Both scores are positive, so clipping preserves the reversal. Coincident within-group representations simplify the construction. If the distance is continuous in the representations, sufficiently small feature perturbations preserve the strict gap.
 This example establishes the possibility of a ranking reversal, not an accuracy improvement or a guarantee about individual final replay coefficients.

\begin{lemma}[Smoothing before clipping]
\label{lem:smoothing-before-clipping}
Averaging signed increments before clipping cannot exceed clipping first and then averaging. For every row-stochastic nonnegative matrix $S$ and every $\delta\in\mathbb{R}^m$,
\begin{equation}
[S\delta]_+\leq S[\delta]_+.
\label{eq:smoothing-before-clipping}
\end{equation}
\end{lemma}

\begin{proof}
The map $x\mapsto[x]_+$ is convex. For every row $i$, Jensen's inequality gives
\begin{equation}
\left[\sum_jS_{ij}\delta_j\right]_+
\leq
\sum_jS_{ij}[\delta_j]_+.
\end{equation}
Applying this coordinatewise proves the claim.
\end{proof}

\subsection{Proof of Proposition~\ref{prop:kernel_estimator}}

\begin{proof}
All expectations are conditional on $\mathcal{G}_t$. Using $\delta_j=F_j+\xi_j$ and $\sum_jS_{ij}=1$,
\begin{equation}
\tilde\delta_i-F_i
=
\sum_jS_{ij}(F_j-F_i)
+
\sum_jS_{ij}\xi_j.
\label{eq:estimation-decomposition}
\end{equation}
The second term has conditional mean zero. Assumption~\ref{ass:spatial-coherence} therefore gives
\begin{equation}
\begin{aligned}
|\mathbb{E}_t[\tilde\delta_i]-F_i|
&\leq
\sum_jS_{ij}|F_j-F_i|\\
&\leq
L_F\sum_jS_{ij}d(h_i,h_j)
=L_Fr_i,
\end{aligned}
\end{equation}
proving the first line of Eq.~(\ref{eq:kernel-estimator-bounds}).

For every $u\in\mathbb{R}$, conditional independence gives
\begin{equation}
\begin{aligned}
\mathbb{E}_t\!\left[
\exp\!\left(u\sum_jS_{ij}\xi_j\right)
\right]
&=
\prod_j\mathbb{E}_t[e^{uS_{ij}\xi_j}]\\
&\leq
\exp\!\left(
\frac{\sigma^2u^2}{2}\sum_jS_{ij}^2
\right).
\label{eq:smoothed-noise-mgf}
\end{aligned}
\end{equation}
Thus, the smoothed residual has variance at most $\sigma^2\sum_jS_{ij}^2$. The bias and centered noise terms in Eq.~(\ref{eq:estimation-decomposition}) have zero cross term, so
\begin{equation}
\mathbb{E}_t[(\tilde\delta_i-F_i)^2]
\leq
L_F^2r_i^2
+
\sigma^2\sum_jS_{ij}^2.
\label{eq:pre-effective-mass-mse}
\end{equation}
Because $0<K_{ij}\leq1$, $K_{ii}=1$, and $\sum_jK_{ij}=1+n_i$,
\begin{equation}
\sum_jS_{ij}^2
=
\frac{\sum_jK_{ij}^2}{(1+n_i)^2}
\leq
\frac{1}{1+n_i}.
\label{eq:effective-mass-bound}
\end{equation}
Equation~(\ref{eq:pre-effective-mass-mse}) proves the second line of Eq.~(\ref{eq:kernel-estimator-bounds}); Eq.~(\ref{eq:effective-mass-bound}) gives the additional $1/(1+n_i)$ bound. Finally, the positive-part map is nonexpansive,
\begin{equation}
|[x]_+-[y]_+|\leq|x-y|,
\end{equation}
so the third line of Eq.~(\ref{eq:kernel-estimator-bounds}) follows by squaring and taking conditional expectation.
\end{proof}

\paragraph{Exact comparison with the raw signed estimator.}
Let $v_j=\mathbb E_t[\xi_j^2]$ and $b_i=\sum_jS_{ij}(F_j-F_i)$. Conditional independence and centering imply
\begin{equation}
\begin{aligned}
\mathbb E_t[(\widetilde\delta_i-F_i)^2]&=b_i^2+\sum_jS_{ij}^2v_j,\\
\mathbb E_t[(\delta_i-F_i)^2]&=v_i.
\end{aligned}
\label{eq:exact-consensus-mse}
\end{equation}
With common variance $v$, subtracting gives the condition  $b_i^2<v(1-\|S_{i\cdot}\|_2^2)$. The nonexpansiveness of clipping bounds its estimation error but does not imply that this condition is necessary and sufficient for comparing two clipped estimators.

\section{Proofs for transport allocation}
\label{app:transport-proofs}

Equation~(\ref{eq:soft-transport-allocation}) is a one-sided transport problem because it constrains every source-column mass but leaves the target marginal free. For finite costs, $\tau>0$, and $\epsilon_T>0$, its feasible set is compact and its entropy-regularized objective is continuous under $0\log0=0$. Strict concavity makes the maximizer unique.

\subsection{Derivation of the transport gradient and pullback}

\begin{proof}
Set $p_{i\mid j}=m\Pi_{ij}$ and let $\Delta_m=\{p\in\mathbb R_+^m:\sum_i p_i=1\}$. For each source column $j$, Eq.~(\ref{eq:soft-transport-allocation}) reduces, up to the factor $1/m$, to
\begin{equation}
\max_{p_{\cdot\mid j}\in\Delta_m}
\left\{
\sum_i p_{i\mid j}(v_i-\tau C_{ij})
-
\tau\epsilon_T\sum_i p_{i\mid j}\log(mp_{i\mid j})
\right\}.
\label{eq:column-transport-objective}
\end{equation}
The entropy-regularized objective is strictly concave on the simplex. With finite scores and positive entropy coefficient, its maximizing probabilities are positive: moving mass into a zero coordinate has a positive, unbounded entropy derivative. The interior Lagrangian first-order condition therefore yields
\begin{equation}
p_{i\mid j}^*(v)
=\frac{\exp((v_i-\tau C_{ij})/(\tau\epsilon_T))}
{\sum_k\exp((v_k-\tau C_{kj})/(\tau\epsilon_T))}.
\end{equation}
Using $\Pi_{ij}^*(v)=p_{i\mid j}^*(v)/m$ gives Eq.~(\ref{eq:soft-transport-plan}). Substitution gives the transport value in log-sum-exp form
\begin{equation}
\mathcal{V}_{\mathrm{OT}}(v)
=
\frac{\tau\epsilon_T}{m}
\sum_{j=1}^{m}
\log\left[
\frac{1}{m}\sum_i\exp((v_i-\tau C_{ij})/(\tau\epsilon_T))
\right].
\label{eq:soft-transport-logsumexp}
\end{equation}
The transport value is smooth, convex, coordinatewise nondecreasing, and translation equivariant:
\begin{equation}
\mathcal{V}_{\mathrm{OT}}(v+c\mathbf{1})
=
\mathcal{V}_{\mathrm{OT}}(v)+c.
\label{eq:transport-envelope-translation}
\end{equation}
Differentiating the log-sum-exp expression gives
\begin{equation}
\begin{aligned}
\frac{\partial\mathcal{V}_{\mathrm{OT}}(v)}{\partial v_i}
&=
\frac{1}{m}\sum_j
\frac{\exp((v_i-\tau C_{ij})/(\tau\epsilon_T))}
{\sum_k\exp((v_k-\tau C_{kj})/(\tau\epsilon_T))}\\
&=
\sum_j\Pi_{ij}^*(v)
=q_i^*(v),
\end{aligned}
\end{equation}
which establishes the target-marginal sensitivity used in Eq.~(\ref{eq:transport-pullback}).

To differentiate through consensus and clipping, define locally $\chi_i=\mathbf{1}\{\tilde\delta_i>0\}$, taking the zero subgradient when $\tilde\delta_i=0$. For any $\delta'$, convexity of the positive-part map gives the following coordinatewise inequality:
\begin{equation}
[S\delta']_+-[S\delta]_+
\geq
\chi\odot S(\delta'-\delta).
\label{eq:relu-subgradient-inequality}
\end{equation}
Since $q^*(\tilde s)=\nabla\mathcal{V}_{\mathrm{OT}}(\tilde s)\geq0$, convexity of $\mathcal{V}_{\mathrm{OT}}$ implies
\begin{equation}
\begin{aligned}
\mathcal V_{\mathrm{OT}}([S\delta']_+)-\mathcal V_{\mathrm{OT}}([S\delta]_+)
&\geq
\langle q^*(\tilde s),[S\delta']_+-[S\delta]_+\rangle\\
&\geq
\left\langle
S^\top(q^*(\tilde s)\odot\chi),
\delta'-\delta
\right\rangle.
\end{aligned}
\end{equation}
This proves that Eq.~(\ref{eq:transport-pullback}) is a valid subgradient (and the gradient away from clipping boundaries). Nonnegativity follows from $S,q^*,\chi\geq0$, and row-stochasticity gives
\begin{equation}
\begin{aligned}
\|w^{\mathrm{SC}}\|_1
&=
\mathbf{1}^\top S^\top(q^*(\tilde s)\odot\chi)\\
&=
\sum_iq_i^*(\tilde s)\chi_i
\leq
\sum_iq_i^*(\tilde s)=1,
\end{aligned}
\end{equation}
This proves the nonnegativity and mass bound used in Eq.~(\ref{eq:blended-weight-mass}).
\end{proof}

\subsection{A conditional regional-mass bound}

\begin{proposition}[Source-local regional mass bound]
\label{prop:regional-transport}
Let $A$ contain $k$ samples, with $1\leq k<m$, and write $q^*(A)=\sum_{i\in A}q_i^*(\widetilde s)$. Suppose there is $\Delta_A>0$ such that
\begin{equation}
\min_{i\in A}(\widetilde s_i-\tau C_{ij})
\geq
\max_{\ell\notin A}(\widetilde s_\ell-\tau C_{\ell j})+\Delta_A
\end{equation}
for every source $j\in A$. Then
\begin{equation}
q^*(A)
\geq
\frac{k}{m}
\left[
1+\frac{m-k}{k}
\exp\!\left(-\frac{\Delta_A}{\tau\epsilon_T}\right)
\right]^{-1}.
\label{eq:regional-transport-lower-bound}
\end{equation}
\end{proposition}

The premise concerns only sources inside $A$. For finite $\Delta_A$, the displayed lower bound is smaller than $k/m$, so it does not establish that the total target mass exceeds uniform replay. Nor does it establish a lower bound on the final coefficients after the signed mask and $S^\top$ pullback. It only quantifies mass supplied by the protected source subset under the stated adjusted-score gap.

\begin{proof}
Fix $j\in A$ and write
\begin{equation}
M_j=\max_{\ell\notin A}(\tilde s_\ell-\tau C_{\ell j}).
\end{equation}
By the adjusted-score gap assumed in Proposition~\ref{prop:regional-transport}, every $i\in A$ satisfies
\begin{equation}
\tilde s_i-\tau C_{ij}\geq M_j+\Delta_A.
\end{equation}
Therefore,
\begin{equation}
\begin{aligned}
\sum_{i\in A}
\exp\!\left(\frac{\tilde s_i-\tau C_{ij}}{\tau\epsilon_T}\right)
&\geq
k\exp\!\left(\frac{M_j+\Delta_A}{\tau\epsilon_T}\right),\\
\sum_{\ell\notin A}
\exp\!\left(\frac{\tilde s_\ell-\tau C_{\ell j}}{\tau\epsilon_T}\right)
&\leq
(m-k)\exp\!\left(\frac{M_j}{\tau\epsilon_T}\right).
\end{aligned}
\end{equation}
Hence the conditional mass sent from source $j$ into $A$ satisfies
\begin{equation}
\sum_{i\in A}m\Pi_{ij}^*(\tilde s)
\geq
\left[
1+\frac{m-k}{k}
\exp\!\left(-\frac{\Delta_A}{\tau\epsilon_T}\right)
\right]^{-1}.
\end{equation}
Finally,
\begin{equation}
q^*(A)
=
\frac{1}{m}\sum_{j=1}^{m}
\sum_{i\in A}m\Pi_{ij}^*(\tilde s).
\end{equation}
Retaining only the $k$ nonnegative source-column terms with $j\in A$ proves Eq.~(\ref{eq:regional-transport-lower-bound}).
\end{proof}

\section{Proof and refinements of Theorem~\ref{thm:pipeline-noise}}
\label{app:pipeline-proof}

\begin{proof}
All expectations are conditional on $\mathcal{G}_t$. Let $\zeta=S\xi$ and define, locally in this proof,
\begin{equation}
\begin{aligned}
v_*&=\sigma^2\max_i\sum_jS_{ij}^2,\\
r_*&=\max_i r_i.
\end{aligned}
\label{eq:pipeline-local-notation}
\end{equation}
Equation~(\ref{eq:smoothed-noise-mgf}) implies that every $\zeta_i$ is $\sqrt{v_*}$-sub-Gaussian. Moreover, Assumption~\ref{ass:spatial-coherence} gives
\begin{equation}
\|SF-F\|_\infty
=
\max_i\left|\sum_jS_{ij}(F_j-F_i)\right|
\leq L_Fr_*.
\label{eq:pipeline-bias-bound}
\end{equation}
By monotonicity and translation equivariance of $\mathcal{V}_{\mathrm{OT}}$,
\begin{equation}
|\mathcal{V}_{\mathrm{OT}}(v)-\mathcal{V}_{\mathrm{OT}}(v')|
\leq\|v-v'\|_\infty.
\label{eq:envelope-lipschitz}
\end{equation}
The positive-part map is coordinatewise one-Lipschitz, so
\begin{equation}
\begin{aligned}
\mathcal V_{\mathrm{OT}}([S\delta]_+)
&=
\mathcal{V}_{\mathrm{OT}}([SF+\zeta]_+)\\
&\leq
\mathcal{V}_{\mathrm{OT}}([F]_+)
+L_Fr_*
+\|\zeta\|_\infty.
\label{eq:pipeline-supnorm-bound}
\end{aligned}
\end{equation}
For every $u>0$,
\begin{equation}
\begin{aligned}
\mathbb{E}_t[e^{u\|\zeta\|_\infty}]
&\leq
\sum_i\mathbb{E}_t[e^{u|\zeta_i|}]\\
&\leq
2m\exp(v_*u^2/2).
\end{aligned}
\end{equation}
Jensen's inequality therefore gives
\begin{equation}
\mathbb{E}_t[\|\zeta\|_\infty]
\leq
\frac{\log(2m)}{u}+\frac{v_*u}{2}.
\end{equation}
Optimizing at $u=\sqrt{2\log(2m)/v_*}$ when $v_*>0$ yields
\begin{equation}
\mathbb{E}_t[\|\zeta\|_\infty]
\leq
\sqrt{2v_*\log(2m)}.
\label{eq:expected-smoothed-maximum}
\end{equation}
The case $v_*=0$ is immediate. Substituting Eqs.~(\ref{eq:pipeline-local-notation}) and~(\ref{eq:expected-smoothed-maximum}) into Eq.~(\ref{eq:pipeline-supnorm-bound}) proves Eq.~(\ref{eq:pipeline-noise-control}).
\end{proof}

\begin{corollary}[High-probability and temperature refinements]
\label{cor:pipeline-refinements}
Under the assumptions of Theorem~\ref{thm:pipeline-noise}, the transport-value guarantee also admits a deviation form and an expectation bound depending explicitly on the entropy scale. Let $r_*=\max_i r_i$. For every $\zeta_0\in(0,1)$, with conditional probability at least $1-\zeta_0$,
\begin{equation}
\begin{aligned}
\mathcal V_{\mathrm{OT}}([S\delta]_+)
\leq{}&
\mathcal{V}_{\mathrm{OT}}([F]_+)+L_Fr_*\\
&+
\sigma\sqrt{
2\log\!\left(\frac{2m}{\zeta_0}\right)
\max_i\sum_jS_{ij}^2
}.
\label{eq:pipeline-noise-high-probability}
\end{aligned}
\end{equation}
Furthermore,
\begin{equation}
\begin{aligned}
\mathbb{E}_t[\mathcal V_{\mathrm{OT}}([S\delta]_+)]
\leq{}&
\mathcal{V}_{\mathrm{OT}}([F]_+)+L_Fr_*\\
&+
\min\Bigg\{
\sigma\sqrt{2\log(2m)\max_i\sum_jS_{ij}^2},\\
&\hspace{22mm}
\frac{\sigma^2\max_i\sum_jS_{ij}^2}{2\tau\epsilon_T}
+\tau\epsilon_T\log2
\Bigg\}.
\label{eq:pipeline-temperature-refinement}
\end{aligned}
\end{equation}
\end{corollary}

\begin{proof}
Set $\zeta=S\xi$ and $v_*=\sigma^2\max_i\sum_jS_{ij}^2$. The sub-Gaussian union bound gives
\begin{equation}
\mathbb{P}_t(\|\zeta\|_\infty\geq x)
\leq
2m\exp\!\left(-\frac{x^2}{2v_*}\right).
\end{equation}
When $v_*=0$, the residuals vanish almost surely and the deviation statement is immediate. For $v_*>0$, taking $x=\sqrt{2v_*\log(2m/\zeta_0)}$ in Eq.~(\ref{eq:pipeline-supnorm-bound}) proves Eq.~(\ref{eq:pipeline-noise-high-probability}).

For the temperature-dependent bound, coordinatewise,
\begin{equation}
[(SF)_i+\zeta_i]_+
\leq
[F_i]_++L_Fr_*+|\zeta_i|.
\end{equation}
Monotonicity and translation equivariance yield
\begin{equation}
\mathcal V_{\mathrm{OT}}([S\delta]_+)
\leq
L_Fr_*+
\mathcal{V}_{\mathrm{OT}}([F]_++|\zeta|).
\end{equation}
For a fixed source column $j$, let
\begin{equation}
a_{ij}
=
\exp\!\left(\frac{[F_i]_+-\tau C_{ij}}{\tau\epsilon_T}\right).
\end{equation}
Conditional Jensen's inequality and sub-Gaussianity give
\begin{equation}
\begin{aligned}
\mathbb{E}_t\!\left[
\tau\epsilon_T\log\left(\frac1m\sum_i a_{ij}e^{|\zeta_i|/(\tau\epsilon_T)}\right)
\right]
&\leq
\tau\epsilon_T\log\left(\frac1m\sum_i a_{ij}
\mathbb{E}_t[e^{|\zeta_i|/(\tau\epsilon_T)}]\right)\\
&\leq
\tau\epsilon_T\log\left(\frac1m\sum_i a_{ij}\right)
+
\frac{v_*}{2\tau\epsilon_T}
+\tau\epsilon_T\log2.
\end{aligned}
\end{equation}
Averaging over source columns proves the second branch in Eq.~(\ref{eq:pipeline-temperature-refinement}). The first branch is Theorem~\ref{thm:pipeline-noise}; taking the smaller proves the corollary.
\end{proof}

\section{Local risk descent and its implementation limits}
\label{app:risk-descent-proof}

This appendix varies the candidate update direction while keeping the scoring geometry fixed; it does not optimize the model parameters directly. Evaluate $w^{\mathrm{SC}}$ at $\delta(g_c)$ and define

\begin{equation}
g_{\mathrm{SC}}^v=\sum_iw_i^{\mathrm{SC}}\nabla_\theta\ell(z_i;\theta_t-\eta g_c).
\label{eq:virtual-point-replay-gradient}
\end{equation}

\begin{theorem}[Prospective spatial-risk descent]
\label{thm:risk-descent}
Fix $S$, $C$, the stored supervision, and all scoring hyperparameters as $g$ varies, and write $\Psi(g)=\mathcal V_{\mathrm{OT}}([S\delta(g)]_+)$. If the memory losses are differentiable and no coordinate of $S\delta(g_c)$ is zero, then
\begin{equation}
\begin{aligned}
\nabla_g\Psi(g_c)
&=-\eta g_{\mathrm{SC}}^{v},\\
D\Psi(g_c)
[g_{\mathrm{SC}}^{v}]
&=-\eta\lVert g_{\mathrm{SC}}^{v}\rVert_2^2.
\end{aligned}
\label{eq:exact-risk-descent-direction}
\end{equation}
\end{theorem}

\begin{proof}
Because no coordinate of $S\delta(g_c)$ is zero, $\Psi$ is differentiable at $g_c$. For every $i$,
\begin{equation}
\nabla_g\delta_i(g_c)
=
-\eta\nabla_\theta\ell(z_i;\theta_t-\eta g_c).
\label{eq:increment-jacobian}
\end{equation}
The chain rule and Eq.~(\ref{eq:transport-pullback}) give
\begin{equation}
\begin{aligned}
\nabla_g\Psi(g_c)
&=
\sum_iw_i^{\mathrm{SC}}\nabla_g\delta_i(g_c)\\
&=
-\eta\sum_iw_i^{\mathrm{SC}}
\nabla_\theta\ell(z_i;\theta_t-\eta g_c)\\
&=-\eta g_{\mathrm{SC}}^{v}.
\end{aligned}
\end{equation}
Taking the inner product with $g_{\mathrm{SC}}^{v}$ proves Eq.~(\ref{eq:exact-risk-descent-direction}).
\end{proof}

\begin{corollary}[Finite-step and implemented-gradient extensions]
\label{cor:risk-descent-extensions}

Assume throughout the fixed-state, differentiability, and nonzero-consensus conditions of Theorem~\ref{thm:risk-descent}. The following bounds distinguish a correction in the virtual replay direction from the current-parameter replay direction used in the ideal update.

Let $\Psi(g)=\mathcal V_{\mathrm{OT}}([S\delta(g)]_+)$. If $\Psi$ has an $L_\Psi$-Lipschitz gradient on the segment from $g_c$ to $g_c+\beta g_{\mathrm{SC}}^{v}$, with $L_\Psi>0$ and $\beta>0$, then
\begin{equation}
\Psi(g_c+\beta g_{\mathrm{SC}}^{v})
\leq
\Psi(g_c)
-
\left(\beta\eta-\frac{L_\Psi\beta^2}{2}\right)
\|g_{\mathrm{SC}}^{v}\|_2^2.
\label{eq:finite-step-risk-descent}
\end{equation}
Hence every $0<\beta<2\eta/L_\Psi$ gives strict descent when $g_{\mathrm{SC}}^{v}\neq0$.

If each memory gradient is $L_\ell$-Lipschitz in $\theta$, define
\begin{equation}
\begin{aligned}
g_{\mathrm{SC}}^{0}
&=\sum_iw_i^{\mathrm{SC}}\nabla_\theta\ell(z_i;\theta_t),\\
g_u
&=\frac{1}{m}\sum_i\nabla_\theta\ell(z_i;\theta_t),\\
\bar g
&=(1-\alpha)g_u+\alpha g_{\mathrm{SC}}^{0}.
\end{aligned}
\label{eq:implemented-replay-directions}
\end{equation}
Then
\begin{equation}
D\Psi(g_c)[\bar g]
\leq
-\eta\|g_{\mathrm{SC}}^{v}\|_2
\left[
\alpha\bigl(\|g_{\mathrm{SC}}^{v}\|_2-\eta L_\ell\|g_c\|_2\bigr)
-(1-\alpha)\|g_u\|_2
\right].
\label{eq:blended-risk-descent-condition}
\end{equation}

Thus, the current-parameter blended replay direction is locally descending for $\Psi$ whenever $g_{\mathrm{SC}}^v\neq0$ and the bracket is positive. This conclusion concerns an infinitesimal correction $g_c+\beta\bar g$. The actual replay strength defines a finite correction and requires its own step-size and smoothness conditions; it is not covered by the finite-step bound for $g_{\mathrm{SC}}^v$ without that additional argument.

\end{corollary}

\begin{proof}
The descent lemma and Theorem~\ref{thm:risk-descent} give
\begin{equation}
\begin{aligned}
\Psi(g_c+\beta g_{\mathrm{SC}}^{v})
&\leq
\Psi(g_c)
+\beta\langle\nabla_g\Psi(g_c),g_{\mathrm{SC}}^{v}\rangle
+\frac{L_\Psi\beta^2}{2}\|g_{\mathrm{SC}}^{v}\|_2^2\\
&=
\Psi(g_c)
-
\left(\beta\eta-\frac{L_\Psi\beta^2}{2}\right)
\|g_{\mathrm{SC}}^{v}\|_2^2,
\end{aligned}
\end{equation}
proving Eq.~(\ref{eq:finite-step-risk-descent}).

For the current-point gradient, Lipschitz continuity and the mass bound $\|w^{\mathrm{SC}}\|_1\leq1$ imply
\begin{equation}
\begin{aligned}
\|g_{\mathrm{SC}}^{0}-g_{\mathrm{SC}}^{v}\|_2
&\leq
\sum_iw_i^{\mathrm{SC}}
\|\nabla_\theta\ell(z_i;\theta_t)
-\nabla_\theta\ell(z_i;\theta_t-\eta g_c)\|_2\\
&\leq
\eta L_\ell\|g_c\|_2\sum_iw_i^{\mathrm{SC}}\\
&\leq
\eta L_\ell\|g_c\|_2.
\end{aligned}
\end{equation}
Therefore,
\begin{equation}
\langle g_{\mathrm{SC}}^{v},g_{\mathrm{SC}}^{0}\rangle
\geq
\|g_{\mathrm{SC}}^{v}\|_2^2
-
\eta L_\ell\|g_c\|_2\|g_{\mathrm{SC}}^{v}\|_2,
\end{equation}
and Cauchy--Schwarz gives
\begin{equation}
\langle g_{\mathrm{SC}}^{v},g_u\rangle
\geq
-\|g_{\mathrm{SC}}^{v}\|_2\|g_u\|_2.
\end{equation}
Since $\nabla_g\Psi(g_c)=-\eta g_{\mathrm{SC}}^{v}$ and $\bar g=(1-\alpha)g_u+\alpha g_{\mathrm{SC}}^{0}$, combining the last two inequalities proves Eq.~(\ref{eq:blended-risk-descent-condition}).
\end{proof}

\section{Experimental details and additional results}
\label{app:experiments}

This appendix collects the implementation and run details for the three
experimental domains, the additional vision results, and the definitions
of the component controls.

\subsection{Replay implementation}
\label{app:replay-implementation}

Let $m$ denote the number of scored memories, which can be smaller than
memory capacity $K$. Scoring and aggregation follow
Sections~\ref{subsec:geometry-estimation} and~\ref{subsec:transport-allocation}.
Writing $u_i=1/m$, the normalized sampling probabilities are
\begin{equation}
  p_i=\frac{\bar w_i}{s_w},\qquad
  \bar w_i=(1-\alpha)u_i+\alpha w_i^{\mathrm{SC}},\qquad
  s_w=\sum_i\bar w_i.
  \label{eq:experimental-sampling}
\end{equation}
The full configurations use $\alpha=0.5$. Their replay estimators differ
from Algorithm~\ref{alg:core-controller} as specified below.

\paragraph{Sampling and replay strength.}
Vision ER and ER-ACE draw 32 memories without replacement according to $p$ and multiply their mean replay loss by $s_w$, with replay strength $\lambda=1$. X-DER changes its label-replay and logit-replay draws but retains its own loss coefficients. Instruction tuning normalizes the weights to unit mass and averages selected losses without replacement. Code RL draws eight groups with replacement and also uses a unit-mass mean. Both use $\lambda=1$, without the $s_w$ multiplier. The without-replacement estimators do not use inclusion-probability corrections and are not generally unbiased for Eq. (\ref{eq:geometry-weighted-replay}), even with the $s_w$ multiplier (Appendix~\ref{app:core-implementation}).

\subsection{Continual language-model post-training}
\label{app:instruction-setup}

\paragraph{Model, stream, and memory.}
We use the 15 instruction-formatted classification tasks in Section~\ref{sec:lnt}, with verbalized labels and exact-match evaluation, without synthetic label corruption. T5-Large uses rank-8 LoRA and a single pass of 860 updates with 16 current examples per update. Table~\ref{tab:instruction-config} specifies the scoring and replay configuration. At each pool refresh, sampling without replacement selects $m=\min(512,|\mathcal M|)$ memories. Hence the candidate pool contains the entire available memory when $K=300$; 512 is an upper bound, not a requirement to duplicate examples.

\begin{table}[!htb]
\centering\small
\caption{Instruction-tuning configuration. Distances and allocation scales
are computed on the scored candidate pool.}
\label{tab:instruction-config}
\setlength{\tabcolsep}{4pt}
\renewcommand{\arraystretch}{1.10}
\begin{tabularx}{\linewidth}{@{}p{0.25\linewidth}X@{}}
\toprule
Quantity & Setting \\
\midrule
Representation & Current T5-Large with its LoRA adapters; last decoder
layer, mean-pooled over target-label tokens (1024 dimensions), then
$\ell_2$-normalized. Refreshed with the candidate pool. \\
Distance & $C_{ij}=\|h_i-h_j\|_2/\nu$, where
$\nu=\mathrm{median}_{k<l}\|h_k-h_l\|_2$ within the pool. \\
Allocation & $\epsilon_s=0.2$, $\epsilon_T=0.5$;
$\tau=c_\beta\,\mathrm{std}([S\delta]_+)/\epsilon_T$ with $c_\beta=1$;
$\alpha=0.5$. \\
Optimization & AdamW learning rate $10^{-3}$ for training;
one virtual SGD step on LoRA parameters with $\eta_v=10^{-2}$. \\
Memory and scoring & $K\in\{300,3000\}$; up to 512 distinct candidates,
refreshed every four replay events. \\
Replay & Four examples per update, without replacement; unit-mass
mean replay loss, $\lambda=1$. \\
\bottomrule
\end{tabularx}
\end{table}

\paragraph{Task orders and paired runs.}
Table~\ref{tab:instruction-orders} gives the three evaluated orders.
A run is an order--seed pair. 
\begin{table}[!htb]
\centering\small
\caption{Task orders for instruction-tuning evaluation.}
\label{tab:instruction-orders}
\setlength{\tabcolsep}{4pt}
\renewcommand{\arraystretch}{1.10}
\begin{tabularx}{\linewidth}{@{}lX@{}}
\toprule
Order & Task sequence \\
\midrule
0 & mnli, cb, wic, copa, qqp, boolq, rte, imdb, yelp, amazon, sst2,
dbpedia, agnews, multirc, yahoo \\
1 & multirc, boolq, wic, mnli, cb, copa, qqp, rte, imdb, sst2, dbpedia,
agnews, yelp, amazon, yahoo \\
2 & yelp, amazon, mnli, cb, copa, qqp, rte, imdb, sst2, dbpedia, agnews, yahoo, multirc, boolq, wic \\
\bottomrule
\end{tabularx}
\end{table}

\subsection{Noisy-label vision}
\label{app:vision-setup}

\subsubsection{Training protocol and host integrations}
\label{app:vision-training}

Split MNIST and Split CIFAR-10 each contain five two-class tasks with
memory capacity 500. Split CIFAR-100 contains ten ten-class tasks with
capacity 2,000; Split TinyImageNet contains ten twenty-class tasks with
capacity 4,000. These are total capacities, not per-class allocations.
Symmetric corruption replaces a selected training label uniformly with
a different class label. Evaluation uses clean labels and, for Class-IL,
no task identifier. Noise rates are 0, 20, 40, and 60\%; MNIST also
includes 80\%.

The CIFAR experiments use one stream pass, ResNet-18 trained from
scratch with SGD, current and replay batches of 32, random crops, and
horizontal flips. The vision runs use an Intel Core i9-13900HX CPU
(24 cores) and an NVIDIA A100 GPU (40 GB).

\paragraph{Host policies and optimization budgets.}
ER and ER-ACE use reservoir memories. ER-ACE retains its asymmetric
stream cross-entropy and cosine classifier. MIR-prop samples in
proportion to positive interference, using candidate pools of 96 on
CIFAR-10 and 160 on CIFAR-100. X-DER retains task-end memory insertion,
logit updates and transplanting, and its auxiliary contrastive draw;
SPHERE replaces its label-replay and logit-distillation draws.

ER-type CIFAR methods use learning rate 0.1; DER++ and X-DER use 0.03.
Their respective distillation/label-loss coefficient pairs are
$(0.5,0.5)$ and $(0.1,0.5)$ on CIFAR-10, and $(0.3,1.0)$ and
$(0.1,0.8)$ on CIFAR-100. GDumb fits its memory for 256 epochs.
The CIFAR-10 PuriDivER configuration uses 255 epochs of MixMatch
memory fitting with CutMix and learning rate 0.03. The evaluated
ABS-based configuration combines ER-ACE, ABS sample selection, and
the same fitting stage. These methods are grouped separately because
additional fitting changes the optimization budget.

\subsubsection{Representation geometry and parameter selection}
\label{app:vision-geometry}

The image representation is the current ResNet-18's globally averaged
last residual-stage output (512 dimensions). MNIST instead uses the
current MLP's last hidden layer (100 dimensions). At each rescoring,
features are recomputed in evaluation mode on stored, unaugmented
inputs and $\ell_2$-normalized. Distances are
$C_{ij}=\|h_i-h_j\|_2^2$. The entire available memory is scored at
the stated interval and at the first step of each task.

Let $r_i(k)$ denote the distance from memory $i$ to its $k$-th
nearest neighbor. The bandwidths and transport temperature are
\begin{equation}
\begin{aligned}
 \epsilon_s &= \mathrm{median}_i\,r_i(k_s),\\
 \epsilon_T &= \max\{\mathrm{median}_i\,r_i(k_T),\epsilon_s\},\\
 T &= T_{\mathrm{frac}}\,
       \mathrm{EMA}_{0.9}\!\left(P_{95}([\delta]_+)\right),
 \qquad \tau=T/\epsilon_T.
\end{aligned}
\label{eq:vision-allocation-scales}
\end{equation}
The exponential moving average is updated at rescorings, and $P_{95}$
is the 95th percentile. Thus $T=\tau\epsilon_T$ is the entropy scale
in Eq.~\ref{eq:soft-transport-allocation}.
Table~\ref{tab:vision-hyperparameters} gives the host-specific settings.
Each virtual update is one SGD step on all parameters at the host
learning rate; parameters and batch-normalization statistics are
restored afterward.

\begin{table}[!htb]
\centering\small
\caption{Vision allocation settings. The last two columns give the
rescoring interval in stream steps and the virtual-step size.
$\alpha=0.5$; ER and ER-ACE use $\lambda=1$.}
\label{tab:vision-hyperparameters}
\setlength{\tabcolsep}{4pt}
\begin{tabular*}{\linewidth}{@{\extracolsep{\fill}}llrrrrr@{}}
\toprule
Dataset & Host & $k_s$ & $k_T$ & $T_{\mathrm{frac}}$ & Interval & $\eta$ \\
\midrule
CIFAR-10, MNIST & ER & 5 & 25 & 0.5 & 4 & 0.1 \\
CIFAR-10 & ER-ACE & 10 & 50 & 1.0 & 4 & 0.1 \\
CIFAR-10 & X-DER & 5 & 50 & 0.5 & 8 & 0.03 \\
CIFAR-100 & ER, ER-ACE & 5 & 25 & 0.5 & 4 & 0.1 \\
CIFAR-100 & X-DER & 5 & 25 & 0.5 & 4 & 0.03 \\
TinyImageNet & ER & 5 & 25 & 0.5 & 4 & 0.1 \\
TinyImageNet & ER-ACE & 5 & 50 & 0.25 & 8 & 0.1 \\
TinyImageNet & X-DER & 5 & 25 & 0.5 & 4 & 0.03 \\
MNIST & ER-ACE & 5 & 25 & 0.5 & 4 & 0.1 \\
\bottomrule
\end{tabular*}
\end{table}

\paragraph{Development search.}
The search used $k_s\in\{5,10\}$, $k_T\in\{25,50\}$,
$T_{\mathrm{frac}}\in\{0.25,0.5,1.0\}$,
rescoring interval in $\{2,4,8\}$, and
$\alpha\in\{0,0.5,0.7,0.85\}$. Selection maximized the mean
improvement over the corresponding $\alpha=0$ control on development
runs. One configuration per dataset and host was fixed before
evaluation. CIFAR-100 ER transfers the CIFAR-10 configuration;
CIFAR-100 ER-ACE selects among twelve transferred configurations.
Baseline searches used the same development seeds: learning rates
$\{0.03,0.1\}$ for ER, ER-ACE, DER++, and X-DER; MIR-prop candidate
sizes $\{96,160,320\}$; and distillation/label-loss coefficients
$\{0.1,0.2,0.3,0.5\}\times\{0.5,0.8,1.0\}$ for DER++ and X-DER.
These host loss coefficients are distinct from SPHERE's mixing
coefficient $\alpha$.

\paragraph{Repeated runs and pairing.}
Host/+SPHERE comparisons match class orders, corruption realizations,
stream orders, and seeds; reservoir-based hosts also match reservoir
decisions. Matching seeds does not require identical model-dependent
memory contents or stored logits for every host. The vision comparison
tables use population standard deviations across evaluation seeds.
The clean neighborhood-block ablation in
Table~\ref{tab:cifar-neighborhood-ablation} instead uses ten paired
runs and sample standard deviations. These evaluation
seeds are separate from the development seeds.

\subsubsection{Additional accuracy and retention results}
\label{app:vision-results}

Tables~\ref{tab:mnist-main} and~\ref{tab:c10-main} supplement
Table~\ref{tab:cifar-main} with final Class-IL accuracy on MNIST and
CIFAR-10. Improvements depend on the host and noise level: the
CIFAR-10 integrations improve ER, ER-ACE, and X-DER at 40\% and 60\%
noise, whereas X-DER's mean decreases at 20\%.

\begin{table}[!htb]
\centering\small
\caption{Split MNIST Class-IL accuracy (\%, $K=500$; mean$\pm$SD).}
\label{tab:mnist-main}
\setlength{\tabcolsep}{4pt}
\begin{tabular*}{\linewidth}{@{\extracolsep{\fill}}lccccc@{}}
\toprule
Method & Clean & 20\% & 40\% & 60\% & 80\% \\
\midrule
\multicolumn{6}{l}{\emph{Stream-time replay}} \\
ER & 87.3$\pm$0.8 & 65.3$\pm$4.0 & 46.9$\pm$5.4 & 31.3$\pm$5.8 & 19.2$\pm$2.9 \\
ER + SPHERE & 86.3$\pm$1.1 & 70.8$\pm$6.7 & 54.1$\pm$4.9 & 31.5$\pm$5.5 & 19.7$\pm$3.1 \\
MIR & 88.2$\pm$0.6 & 61.6$\pm$6.6 & 45.7$\pm$5.0 & 31.5$\pm$5.5 & 20.6$\pm$2.5 \\
DER++ & \textbf{90.7$\pm$0.6} & 77.5$\pm$6.0 & 62.1$\pm$4.6 & 38.0$\pm$3.5 & 20.7$\pm$3.0 \\
ER-ACE & 88.8$\pm$1.1 & 75.7$\pm$9.4 & 66.2$\pm$2.8 & 43.7$\pm$2.9 & \textbf{21.3$\pm$2.8} \\
ER-ACE + SPHERE & 89.2$\pm$0.4 & \textbf{81.6$\pm$3.8} & \textbf{69.8$\pm$3.2} & \textbf{45.8$\pm$4.5} & \textbf{21.3$\pm$2.1} \\
\bottomrule
\end{tabular*}
\end{table}

\begin{table}[!htb]
\centering\small
\caption{Final Class-IL accuracy on Split CIFAR-10 (\%, memory 500).
Values are mean$\pm$SD. Best means are bold within each protocol.}
\label{tab:c10-main}
\setlength{\tabcolsep}{3.5pt}
\begin{tabular*}{\linewidth}{@{\extracolsep{\fill}}lcccc@{}}
\toprule
Method & Clean & 20\% & 40\% & 60\% \\
\midrule
\multicolumn{5}{l}{\emph{Stream-time replay}} \\
ER & 43.4$\pm$5.7 & 33.4$\pm$4.5 & 25.1$\pm$2.9 & 17.4$\pm$2.7\\
ER + SPHERE & 47.6$\pm$3.8 & 35.8$\pm$2.7 & 27.6$\pm$1.6 & 19.9$\pm$1.8\\
MIR-prop & 44.6$\pm$5.2 & 32.1$\pm$4.4 & 24.7$\pm$2.1 & 18.1$\pm$1.2\\
DER++ & 43.7$\pm$5.1 & 28.2$\pm$4.3 & 23.3$\pm$4.8 & 18.2$\pm$3.3\\
ER-ACE & 39.8$\pm$4.5 & 36.1$\pm$4.5 & 27.9$\pm$3.4 & 20.3$\pm$2.3\\
ER-ACE + SPHERE & 41.7$\pm$5.8 & \textbf{38.8$\pm$4.2} & \textbf{31.3$\pm$2.4} & \textbf{22.6$\pm$1.8}\\
X-DER & 45.4$\pm$2.9 & 28.1$\pm$8.2 & 25.1$\pm$6.4 & 17.2$\pm$3.9\\
X-DER + SPHERE & \textbf{48.6}$\pm$4.4 & 27.5$\pm$8.9 & 26.6$\pm$6.9 & 18.4$\pm$3.7\\
\midrule
\multicolumn{5}{l}{\emph{Separate buffer training}} \\
GDumb & 46.6$\pm$1.3 & 36.8$\pm$0.9 & 27.2$\pm$0.9 & 19.8$\pm$1.1\\
PuriDivER & 62.2$\pm$0.9 & 60.6$\pm$1.6 & 57.4$\pm$1.3 & 45.1$\pm$1.8\\
PuriDivER + SPHERE & \textbf{65.5$\pm$1.7} & \textbf{62.0$\pm$1.5} & \textbf{57.9$\pm$2.4} & \textbf{51.3$\pm$4.3}\\
ABS-based & 48.5$\pm$2.9 & 47.6$\pm$2.5 & 42.4$\pm$2.3 & 29.1$\pm$2.1\\
\bottomrule
\end{tabular*}
\end{table}

Table~\ref{tab:cifar-retention-extra} combines the 40\% and 60\%
retention results; Table~\ref{tab:vision-retention-main} gives 20\%.
Class-IL and Task-IL use final accuracy without and with task identity.
AAA averages incremental Class-IL accuracy. Average forgetting is
computed per class; worst-class accuracy is the minimum final class
accuracy, and $\mathrm{CVaR}_{0.2}$ averages forgetting over the worst
20\% of classes. ER-ACE + SPHERE improves all six reported means over
ER-ACE at both additional noise levels, including the worst-class and
tail metrics.

\begin{table}[!htb]
\centering\footnotesize
\caption{CIFAR-10 retention (\%, $K=500$; mean$\pm$SD).
Bold compares ER-ACE with its SPHERE integration.}
\label{tab:cifar-retention-extra}
\setlength{\tabcolsep}{3pt}
\renewcommand{\arraystretch}{1.10}
\begin{tabular*}{\linewidth}{@{\extracolsep{\fill}}lcccccc@{}}
\toprule
Method & Class-IL$\uparrow$ & Task-IL$\uparrow$ & AAA$\uparrow$ & avgF$\downarrow$ & Worst-class$\uparrow$ & CVaR$_{0.2}\downarrow$ \\
\midrule
\multicolumn{7}{l}{\emph{40\% symmetric noise}} \\
ER & 25.1$\pm$2.8 & 71.6$\pm$3.4 & 45.3$\pm$3.6 & 49.7$\pm$5.1 & 2.6$\pm$1.5 & 81.1$\pm$3.5 \\
ER-ACE & 27.9$\pm$3.4 & 72.7$\pm$4.1 & 47.0$\pm$3.3 & 28.4$\pm$4.5 & 8.3$\pm$6.0 & 59.4$\pm$8.6 \\
ER-ACE + SPHERE & \textbf{31.3$\pm$2.4} & \textbf{74.3$\pm$5.4} & \textbf{50.2$\pm$3.0} & \textbf{18.4$\pm$6.4} & \textbf{16.4$\pm$3.6} & \textbf{52.2$\pm$8.0} \\
DER++ & 23.3$\pm$4.8 & 76.8$\pm$4.7 & 39.7$\pm$4.9 & 60.6$\pm$9.5 & 2.0$\pm$4.6 & 86.9$\pm$7.5 \\
MIR-prop & 24.7$\pm$2.1 & 69.3$\pm$3.8 & 38.7$\pm$2.4 & 37.9$\pm$5.2 & 2.5$\pm$2.5 & 72.8$\pm$5.3 \\
X-DER & 25.1$\pm$6.4 & 71.7$\pm$8.8 & 38.4$\pm$7.8 & 24.3$\pm$8.0 & 0.0$\pm$0.0 & 59.1$\pm$14.8 \\
\addlinespace[2pt]
GDumb (buffer fitting) & 27.2$\pm$0.8 & 69.0$\pm$2.7 & 26.2$\pm$0.7 & 25.0$\pm$2.8 & 7.3$\pm$3.8 & 59.6$\pm$16.3 \\
\midrule
\multicolumn{7}{l}{\emph{60\% symmetric noise}} \\
ER & 17.4$\pm$2.7 & 62.4$\pm$3.0 & 36.4$\pm$3.6 & 54.8$\pm$4.5 & 0.8$\pm$1.2 & 84.5$\pm$6.9 \\
ER-ACE & 20.3$\pm$2.2 & 64.8$\pm$4.1 & 39.3$\pm$2.8 & 27.0$\pm$5.3 & 5.1$\pm$3.2 & 58.2$\pm$7.6 \\
ER-ACE + SPHERE & \textbf{22.6$\pm$1.8} & \textbf{67.4$\pm$2.7} & \textbf{41.9$\pm$2.9} & \textbf{21.0$\pm$6.5} & \textbf{11.2$\pm$1.6} & \textbf{52.0$\pm$9.9} \\
DER++ & 18.2$\pm$3.3 & 68.8$\pm$3.4 & 34.2$\pm$3.3 & 62.2$\pm$7.3 & 0.3$\pm$0.4 & 90.1$\pm$4.9 \\
MIR-prop & 18.1$\pm$1.2 & 63.8$\pm$2.1 & 32.2$\pm$2.5 & 33.0$\pm$4.1 & 2.0$\pm$1.1 & 73.9$\pm$4.6 \\
X-DER & 17.2$\pm$3.9 & 63.8$\pm$5.8 & 33.7$\pm$5.0 & 25.5$\pm$4.9 & 0.0$\pm$0.0 & 64.8$\pm$8.7 \\
\addlinespace[2pt]
GDumb (buffer fitting) & 19.8$\pm$1.1 & 62.9$\pm$2.0 & 24.7$\pm$0.8 & 22.1$\pm$2.8 & 3.8$\pm$2.8 & 67.0$\pm$18.3 \\
\bottomrule
\end{tabular*}
\end{table}

\subsection{Code-generation reinforcement learning}
\label{app:code-setup}

\subsubsection{Data and optimization}

The four KodCode-V1 stages are Package, Data Structure/\allowbreak Algorithm,
Leetcode, and Codeforces/\allowbreak Apps. Problems use only the standard library
and have 4--12 collectible pytest nodes.
Reference solutions must pass all nodes in two executions; unstable
nodes are excluded. Near duplicates are removed using centered prompt
embeddings with cosine similarity above 0.90. Each stage has 120
training, 100 evaluation, 40 development-fitting, and 40
development-probing problems, with evaluation disjoint from both
training and development. The visible test is the first node after
sorting by the SHA256 hash of problem and node identifiers. Full test
suites are unavailable to training and replay selection and are used
only for offline audits and evaluation.

Qwen2.5-Coder-1.5B-Instruct uses rank-16 LoRA on all projection layers,
with LoRA scaling parameter 32 (not SPHERE's $\alpha$). AdamW uses
learning rate $10^{-4}$. Eight current problems and two passes over
each stage give 30 updates per stage and 120 overall. The protocol permits
repeated access within a stage. Eight programs are generated per problem by
vLLM using the current LoRA version. Prompts and responses are capped
at 1,536 and 512 tokens. Programs run in an isolated environment with
CPU, memory, and wall-clock limits and pytest plugins disabled.

The group advantages are centered without division by reward standard
deviation. The clipped token-level surrogate uses a fixed denominator
of 512 and sampler correction
$\min\{2,\exp(\log p_0-\log q)\}$, where $p_0$ and $q$ are the
trainer and sampler probabilities. A virtual AdamW step on current
problems, at learning rate $10^{-4}$, scores candidate groups with the
same sampled programs and advantages before and after the step.
Parameters and optimizer state are then restored.

\subsubsection{Geometry, candidates, and development selection}

Prompt embeddings come from the frozen base model's last hidden layer,
mean-pooled over prompt tokens (1,536 dimensions). They are
$\ell_2$-normalized, centered by subtracting the mean over all training
problems, and normalized again. Let $d_0=0.524$ be the median
cosine distance to the eighth nearest neighbor over the training
pool. The geometry and allocation settings are $C_{ij}=\frac{1-h_i^\top h_j}{d_0}$, $\epsilon_s=\frac{1}{\ln 10}$, $\epsilon_T=1$, $\tau=0.25\,P_{95}([S\delta]_+)$ and $\alpha=0.5$.

Embeddings and their normalization are computed before training and
held fixed. This preprocessing uses training prompts from all four
stages, but no evaluation problems or full-suite outcomes.

Reservoir admission follows a pregenerated schedule shared within each
seed. Capacity is 480, equal to the number of distinct training problems;
scoring and replay budgets, rather than this capacity, restrict access.
Candidate selection favors problems whose latest replay group had both
passing and failing programs. Each update generates eight programs
for each of 48 candidates, and only mixed-outcome groups enter the
scored set. SPHERE replays eight groups with replacement and unit
total weight. All replay methods share this candidate-generation
budget. Common random seeds do not force identical candidates or
programs after model-dependent histories diverge. No-replay training
generates no candidates.

\subsubsection{Evaluation metrics}
\label{app:code-metrics}

Greedy evaluation requires passing the full test suite.
Let $A_{t,j}$ be full-suite accuracy in percent on stage $j$ after
training stage $t$, and $A_{0,j}$ its initial accuracy. For four stages,
\begin{equation}
 \operatorname{Final}=\frac14\sum_{j=1}^4 A_{4,j},\qquad
 \operatorname{Old}=\frac13\sum_{j=1}^3 A_{4,j},\qquad
 \operatorname{Worst}=\min_{1\leq j\leq4}A_{4,j}.
 \label{eq:code-evaluation-metrics}
\end{equation}
\begin{equation}
\begin{aligned}
 F&=\frac13\sum_{j=1}^3
       \left(\max_{j\leq t\leq4}A_{t,j}-A_{4,j}\right),\\
 \operatorname{BWT}&=\frac13\sum_{j=1}^3(A_{4,j}-A_{j,j}).
\end{aligned}
\label{eq:code-retention-metrics}
\end{equation}
Forgetting and BWT are measured in percentage points;
lower forgetting and higher BWT are better. Checkpoint-average
accuracy, when used, is $(1/16)\sum_{t=1}^4\sum_{j=1}^4 A_{t,j}$.
It includes future stages at earlier checkpoints and is not a
seen-task-only incremental average. 

\subsection{Component ablations and interpretation}
\label{app:ablations}

\begin{table}[h]
\centering\small
\caption{Instruction-tuning controls in
Tables~\ref{tab:sphere-neighborhood-ablation} and~\ref{tab:lnt-ablation}.
The two transport-free controls use different allocation rules.}
\label{tab:ablation-definitions}
\setlength{\tabcolsep}{4pt}
\renewcommand{\arraystretch}{1.12}
\begin{tabularx}{\linewidth}{@{}p{0.25\linewidth}X@{}}
\toprule
Control & Operation \\
\midrule
$S=I$ & Use the identity for aggregation and omit its pullback;
retain clipping, transport with $C$, and the uniform floor. \\
Consensus only &
Replace $q^\ast(v)$ by $\operatorname{softmax}_T(v)$,
with $T$ chosen to match the ESS of the full method's weights.
Set$w^{\mathrm{SC}}
 = S^\top(\operatorname{softmax}_T(v)\odot\chi)$,
and retain the uniform mixture and normalization in Eq. (\ref{eq:experimental-sampling}).
\\
No transport & Normalize the supported field directly,
$\mathcal N(v)$, without transport or the $S^\top$ pullback. \\
Permuted geometry & At every pool refresh, randomly permute feature
rows before constructing $C$ and $S$, breaking their alignment with
the memories while retaining the full allocation procedure. \\
Matched MIR & Use $\operatorname{softmax}_T([\delta]_+)$ with
$T$ chosen to match the ESS of the full method's weights. \\
No uniform floor & Set $\alpha=1$; remove only the uniform mixture
from the full allocation rule. \\
\bottomrule
\end{tabularx}
\end{table}

The controls test different alternatives to the full allocation rule.
Names such as ``No transport'' do not denote the same operation across
domains. Below, $v=[S\delta]_+$ is the supported field,
$\chi=\mathbf1_{\{S\delta>0\}}$, and $u_i=1/m$.
For a nonnegative vector of positive mass, write
$\mathcal N(a)=a/\sum_i a_i$, and define
$\operatorname{softmax}_T(a)_i=\exp(a_i/T)/\sum_j\exp(a_j/T)$
for $T>0$. The full method uses
$w^{\mathrm{SC}}=S^\top(q^*(v)\odot\chi)$ and
Eq. (\ref{eq:experimental-sampling}).

Table~\ref{tab:ablation-definitions} gives the specified replacement
in each control. All retain the full method's memory capacity,
candidate-pool budget, replay-batch budget, and virtual step.
The two controls called ``Consensus only'' and ``No transport'' differ
in their allocation rule: the former uses an ESS-matched softmax,
whereas the latter directly normalizes the supported scores and omits
the pullback. Their results are therefore not two evaluations of an
identical transport-free method.

\paragraph{Scope of component attribution.}
The $S=I$ control evaluates signed aggregation and its
chain-rule pullback jointly, while retaining the transport
distance cost. It measures the contribution of the neighborhood
construction as a whole, rather than assigning separate gains
to aggregation and $S^\top$.
Consensus only instead compares geometric transport with an
ESS-matched softmax of the same supported scores, retaining
the downstream weighting operations.
These comparisons evaluate specific alternatives to the full
method; their differences are not additive shares of the
overall improvement.
Instruction tuning and code-generation RL both use unit-mass
replay without the $s_w$ multiplier, so their gains do not
depend on that explicit scaling factor.
The no-uniform-floor control tests baseline uniform rehearsal,
not the isolated effect of $s_w$.

\end{document}